%% file: arxiv-main.tex
\documentclass[runningheads]{llncs}
\usepackage{graphicx}
\usepackage{times}
\usepackage{soul}
\usepackage{url}
\usepackage{wrapfig}
\usepackage[hidelinks]{hyperref}
\usepackage[utf8]{inputenc}
\usepackage[small]{caption}
\usepackage{graphicx}
\usepackage{booktabs}
\input{inputs}

\begin{document}
\title{SF-PATE: Scalable, Fair, and Private Aggregation of Teacher Ensembles}
%
%
\author{Cuong Tran\inst{1} \and
Keyu Zhu\inst{2} \and
Ferdinando Fioretto\inst{1} \and
Pascal Van Hentenryck\inst{2}}
%
%
\institute{Syracuse University \\ 
\email{cutran@syr.edu, ffiorett@syr.edu} 
\and
Georgia Institute of Technology \\
\email{kzhu67@gatech.edu, pvh@isye.gatech.edu}
}
\maketitle\sloppy\allowdisplaybreaks              
\begin{abstract}
A critical concern in data-driven processes is to build models whose outcomes do not discriminate against some demographic groups, including gender, ethnicity, or age. To ensure non-discrimination in learning tasks, knowledge of the group attributes is essential. However, in practice, these attributes may not be available due to legal and ethical requirements. To address this challenge, this paper studies a model that protects the privacy of the individuals’ sensitive information while also allowing it to learn non-discriminatory predictors.
A key characteristic of the proposed model is to  enable the adoption of off-the-selves and non-private fair models to create a privacy-preserving and fair model. 
The paper analyzes the relation between accuracy, privacy, and fairness, and the experimental evaluation illustrates the benefits of the proposed models on several prediction tasks. In particular, this proposal is the first to allow both scalable and accurate training of private and fair models for very large neural networks.
\end{abstract}
%
%
%

\section{Introduction}

A number of decision processes, such as criminal assessment, landing, and hiring, are increasingly being aided by machine learning systems.
A critical concern is that the learned models are prone to report outcomes that are discriminatory against some demographic group, including gender, ethnicity, or age. 
These concerns have spurred the recent development of fairness definitions and algorithms for decision-making, focusing attention on the tradeoff between the model accuracy and fairness. 

To ensure non-discrimination in learning tasks, knowledge of the \emph{sensitive} attributes is essential. At the same time, legal and ethical requirements often prevent the use of this sensitive data. For example, U.S.~law prevents using racial identifiers in the development of models for consumer lending or credit scoring. 
Other requirements may be even more stringent, and prevent the collection of protected user attributes, such as for the case of racial attributes in the E.U.~General Data Protection Regulation (GDPR), or require protection of the consumer data privacy. 
In this scenario, an important tension arise between (1) the demand for models to be non-discriminatory, (2) the requirement for such model to use the protected attribute during training, as adopted by common fairness models, and (3) the restriction on the data or protected attributes that can be used. 
There is thus a need to provide learning models that can both guarantee non discriminatory decisions and protect the privacy of the individuals' group attributes.


To this end, this paper introduces a novel differential privacy framework to train deep learning models that satisfy several group fairness notions, including \emph{equalized odds}, \emph{accuracy parity}, and \emph{demographic parity} \cite{zafar:17,hardt:16,agarwal:18}, while providing privacy of the protected attributes.
The proposed framework, called \emph{Scalable, Fair, and Private Aggregation of Teacher Ensemble (SF-PATE)} is inspired by the success obtained by teachers ensemble frameworks in private learning tasks \cite{papernot:16}. These frameworks transfer the classification knowledge learned from a pretrained ensemble of models (called teachers) to a target model (called student) via a privacy-preserving aggregation process. 
This paper exploit this key idea, but rather than transferring the classification capability of the models, it seek to answer a novel question: {\em can fairness properties of a model ensemble be transferred to a target model}? 

In addition to providing an affirmative answer to the question above, the key contributions of this paper can be summarized as follows: 
({\bf 1}) It proposes two flavors of SF-PATE. The first, called $\text{SF}_T$-PATE allows transferring fairness properties of the ensemble 
while protecting the demographic group attributes; the second, called 
$\text{SF}_S$-PATE, in addition, also protects the target labels. 
({\bf 2}) In addition to analyze the models privacy, the paper provides an analysis on the fairness properties of SF-PATE and shows that unfairness can be bounded in many practical settings. 
({\bf 3}) Importantly, SF-PATE allows the adoption of black-box (non-private) fair models during the knowledge transfer process. This feature is unique to the proposed framework and simplifies the development and adoption of new fairness metrics in privacy-preserving ML.
({\bf 4}) Evaluation on both tabular and image datasets show not only that SF-PATE achieves better accuracy, privacy, and fairness tradeoffs with respect to the current state of the art, but it is also significantly faster. These properties render SF-PATE amenable to train large,  overparameterized, models, that ensure privacy, accuracy, and fairness simultaneously. 

To the best of the authors knowledge, the proposed framework is the first to enable this important combination, rendering SF-PATE a promising tool for privacy-preserving and fair decision making.

\section{Related work}
Privacy and fairness have mostly been studied in isolation. A few exceptions are 
represented by the work of Dwork et al.~\cite{dwork:12}, which is one of 
the earliest contribution linking fairness and differential privacy, 
showing that individual fairness is a generalization of differential privacy. 
Cummings et al.~\cite{cummings:19} study the tradeoff between differential privacy and equal opportunity, a notion of fairness that restricts a classifier to produce equal true positive rates across different groups. The work shows that there is no classifier that achieves $(\epsilon,0)$-differential privacy, satisfies equal opportunity, and has accuracy better than a constant classifier.  
A recent line of work has also observed that private models may have a negative impact towards fairness~\cite{bagdasaryan:19,pujol:20}.
Building from these observations, Ekstrand et al.~\cite{ekstrand:18} raise questions about the tradeoff between privacy and fairness and, Jagielski et al.~\cite{jagielski2019differentially} and Mozannar et al.~\cite{mozannar2020fair} propose two simple, yet effective algorithms that satisfy {$(\epsilon, \delta)$ and $\epsilon$- differential privacy, respectively, for equalized odds}. Xu et al.~\cite{xu2019achieving} proposes a private and fair logistic regression model making use of the functional mechanism \cite{JMLR:v12:chaudhuri11a}. 
Finally, Tran \cite{tran2021differentially} proposed \emph{PF-LD}, which trains a classifier under differential privacy stochastic gradient descent \cite{abadi:16} and imposes fairness constraints using a privacy-preserving extension of the Dual Lagrangian framework proposed in \cite{fioretto2020lagrangian}. 
Although this method was shown having substantial improvements on the privacy and fairness trade-offs over existing baselines, it has a high computational cost, due to the operations required to retain both privacy and fairness during training. 

In contrast, this work introduces a semi-supervised framework that relies on transferring privacy and fairness from a model ensemble to construct high-quality private and fair classifiers, while also being scalable.

\section{Problem Settings and Goals}
\label{sec:problem_setting}
The paper considers datasets $D$ consisting of $n$ individual data points $(X_i, A_i, Y_i)$, with $i \in [n]$ drawn i.i.d.~from an unknown distribution $\Pi$. Therein, $X_i \in \mathcal{X}$ is a \emph{non-sensitive} feature vector, $A_i \in \mathcal{A}$, with $\mathcal{A} = [m]$ (for some finite $m$) is a demographic group attribute, and $Y_i \in \mathcal{Y}$ is a class label.
The goal is to learn a classifier $\cM_\theta : \mathcal{X} \to \mathcal{Y}$, where $\theta$ is a vector of real-valued parameters, 
that ensures a specified non-discriminatory notion with respect to $A$ while guaranteeing the \emph{privacy} of the group attribute $A$. 
The model quality is measured in terms of a non-negative, and assumed differentiable, \emph{loss function} $\mathcal{L}: \mathcal{Y} \times \mathcal{Y} \to \mathbb{R}_+$, and the problem is that of minimizing the empirical risk function (ERM):
\begin{equation}
\label{eq:erm}
   \optimal{\theta} \, =  \argmin_\theta J(\cM_\theta, D) = \underset{\theta}{\arg\min}~
   \frac{1}{n} \sum_{i=1}^n 
    \mathcal{L}(\cM_\theta(X_i), Y_i).
    \tag{P}
\end{equation}
\setcounter{equation}{0}
The paper focuses on learning general classifiers, such as neural networks, that satisfy group fairness (as defined next) and protect the disclosure of the group attributes using the notion of differential privacy.
Importantly, the paper assumes that the attribute $A$ is not part of the model input during inference. This is crucial in the application of interest to this work as the protected attributes cannot be disclosed. 


\subsection{Fairness}
\label{sec:fairness}

This work constrains a classifier $\cM$ to satisfy a given group fairness notion under a distribution over $(X, A, Y)$ for the protected attribute $A$. 
The class of fairness constraints considered is specified by the following expression
    \begin{equation}\label{eq:main_obj}
        \left| \mathbb{E}_{X,Y|A=a } [ h(\cM_{\theta}(X),Y)]  -  \mathbb{E}_{X,Y} [ h(\cM_{\theta}(X),Y)] \right|  \leq \alpha\,,  \qquad \forall~a \in [m].
    \end{equation}
\noindent
The above compares a property for a group of individuals with respect to the whole population and quantifies its difference. 
Therein, $h(\cM_{\theta}(X),Y) \in \RR^{c}$, referred here as \emph{fairness function}, defines a target group fairness notion as depicted in more details below. Parameter $\alpha \in \RR_{+}^c$ represents the \emph{fairness violation} and leads to the following fairness metric.
\begin{definition}[$\alpha$-fairness]\label{def:a-fair}
A model $\cM_{\theta}$ is $\alpha$-fair w.r.t.~a joint distribution $(X,A,Y)$ and fairness function $h(\cdot)$ iff:
\begin{equation}
\label{eq:alpha_fairness}
    \xi(D, \theta) = 
    \max_{a \in [m]}~\left\vert\mathbb{E}_{X,Y|A=a} [ h(\cM_{\theta}(X),Y)]  -  \mathbb{E}_{X,Y} [ h(\cM_{\theta}(X),Y)] 
    \right\vert \leq \alpha,
\end{equation}
wehre $\xi(D, \theta)$ is referred to as fairness violation. 
\end{definition}
When a classifier $\cM_{\theta}$ is $0$-fair, we also say that it is \emph{perfectly fair}.
While the above capture an ample class of group fairness notions, the paper focuses on the following three popular notions. 
\begin{itemize}[leftmargin=*, parsep=0pt, itemsep=0pt, topsep=0pt]
	\item[$\bullet$] \emph{Demographic Parity}: 
	$\cM$'s predictions are statistically independent of the protected attribute $A$. That is, for all  $a \in \cA$, $\hat{y} \in \cY$,
	\begin{equation*} 
	    \pr{\cM(X) = \hat{y} \mid A = a}= \pr{\cM(X) = \hat{y}},
	\end{equation*}
	which is expressed by setting $h(\cM_{\theta}(X), Y)$ as $\mathbbm{1}\{\cM_{\theta}(X) =1\}$.
    \item[$\bullet$] \emph{Equalized odds}: $\cM$'s predictions are conditionally independent of the protected attribute $A$ given the label $Y$. 
	That is, for all $a \in \mathcal{A}$, $\hat{y} \in \mathcal{Y}$, and $y \in \mathcal{Y}$,
	\begin{equation*}
    \pr{\cM(X) = \hat{y} \mid A=a, Y=y} = 
    \pr{\cM(X) = \hat{y} \mid Y=y},
	\end{equation*}
	which is expressed by:
\begin{equation*}
    h(\cM_{\theta}(X),Y) = \begin{bmatrix}
  \mathbbm{1}\{\cM_{\theta}(X) =1 ~ \& ~ Y=0 \}    \\ 
  \mathbbm{1}\{\cM_{\theta}(X) =1  ~\&~  Y=1\} 
\end{bmatrix}.
\end{equation*}
	\item[$\bullet$] \textit{Accuracy parity}:
	$\cM$'s mis-classification rate is conditionally independent of the protected attribute.
	That is, for any $a\in \cA$,
	\begin{equation*}
		\pr{\cM(X) \neq Y \mid A=a} = \pr{\cM(X) \neq Y},
	\end{equation*}
	which is captured with $h(\cM_{\theta}(X),Y)=\mathbbm{1}\{\cM_{\theta}(X) \neq Y \}$.
\end{itemize}
Notice that the fairness function $h(\cdot)$ may not be differentiable w.r.t.~$\cM_{\theta}$. Thus, practically,  one considers differentiable surrogates; for example, for the case of demographic parity $h(\cdot) \approx \cM_{\theta}(X)$, and for accuracy parity, $h(\cdot) \approx \cL(\cM_{\theta}(X), Y)$.

\subsection{Differential Privacy} 
\label{sec:dp}
Differential privacy (DP) \cite{dwork:06} is a strong privacy notion used to quantify and bound the privacy loss of an individual participation to a computation. 
Similarly to \cite{jagielski2019differentially,tran2021differentially}, this work focuses on the instance where the protection is restricted to the group attributes only. A dataset $D \in \mathcal{D} = (\mathcal{X} \times \mathcal{A} \times \mathcal{Y})$ of size $n$ can be described as a pair $(D_P, D_S)$ where 
$D_P \in (\mathcal{X} \times \mathcal{Y})^n$ describes the \emph{public} attributes and $D_S \in \mathcal{A}^n$ describes the group attributes. 
\emph{The privacy goal is to guarantee that the output of the learning model does not differ much when a single individual group attribute is changed}.

The action of changing a single attribute from a dataset $D_S$, resulting in a new dataset $D_S'$, defines the notion of \emph{dataset adjacency}. Two dataset $D_S$ and $D_S' \in \mathcal{A}^n$ are said adjacent, denoted $D_S \sim D_S'$, if they differ in at most a single entry (e.g., in one individual's group membership).

\begin{definition}[Differential Privacy]
	A randomized mechanism $\mathcal{M} : \mathcal{D} \to \mathcal{R}$ with domain $\mathcal{D}$ and range $\mathcal{R}$ is $(\epsilon, \delta)$-differentially private w.r.t.~attribute $A$, if, for any dataset  $D_P \in (\mathcal{X} \times \mathcal{Y})^n$, any two adjacent inputs $D_S, D_S' \in \mathcal{A}^n$, and any subset of output responses $R \subseteq \mathcal{R}$:
	\[
	    \pr{\cM(D_P, D_S) \in R }\leq  \exp\left(\epsilon\right)\cdot
	    \pr{\cM(D_P, D_S') \in R } + \delta.
	\]
\end{definition}
\noindent 
When $\delta=0$ the algorithm is said to satisfy $\epsilon$-DP. 
Parameter $\epsilon > 0$ describes the \emph{privacy loss} of the algorithm, 
with values close to $0$ denoting strong privacy, while parameter 
$\delta \in [0,1]$ captures the probability of failure of the algorithm to 
satisfy $\epsilon$-DP. The global sensitivity $\Delta_f$ of a real-valued 
function $f: \mathcal{D} \to \mathbb{R}^k$ is defined as the maximum amount 
by which $f$ changes  in two adjacent inputs $D$ and $D'$:
\(
	\Delta_f = \max_{D \sim D'} \| f(D) - f(D') \|.
\)
In particular, the Gaussian mechanism, defined by
\(
    \mathcal{M}(D) = f(D) + \mathcal{N}(0, \Delta_f^2 \, \sigma^2), 
\)
where $\mathcal{N}(0, \Delta_f\, \sigma^2)$ is 
the Gaussian distribution with $0$ mean and standard deviation 
$\Delta_f\, \sigma^2$, satisfies $(\epsilon, \delta)$-DP for 
$\delta > \frac{4}{5} \exp(-(\sigma\epsilon)^2 / 2)$ \cite{dwork:06}. 

\section{Learning Private and Fair Models: Challenges}
\label{sec:limitation}
When interpreted as constraints of the form \eqref{eq:alpha_fairness}, fairness properties can be explicitly imposed to problem \eqref{eq:erm}, resulting in a constrained empirical risk problem. The resolution of this problem can be attacked using primal-dual solutions, such as the Lagrangian Dual framework proposed by Tran et al.~\cite{tran2021differentially}. However, the privacy analysis becomes convoluted, often requiring the use of approximate solutions to simplifying it (see \cite{tran2021differentially}). More importantly, calibrating the appropriate amount of noise to enforcing DP requires the use of specialized solutions (e.g., to assess the sensitivity associated with the fairness constraints) and thus prevents the adoption of this framework when new fairness notions arise, as the paper shows in the evaluation, ultimately increasing the adoption barrier.  
Finally, these methods are adaptations of celebrated DP-SGD \cite{abadi:16} algorithm, whose private computations rely on sequential clipping operations, and, when used in combination with fairness constraints, produce a very slow training process. Unfortunately, the adoption of fairness constraints inhibits the effectiveness of frameworks like JAX \cite{jax2018github} and Opacus \cite{opacus}, that use vectorization to speed up operations in DP-SGD. 

Alternatively, and as proposed in \cite{mozannar2020fair}, one can adopt a pre-processing step to perturb the group attributes and then use a non-private, fair ML model to post-process the private data. This scheme, however has been shown to introduce an unnecessarily large amount of noise, especially for multigroup ($m > 2$) tasks \cite{tran2021differentially}. 

The approach proposed in this paper avoids these difficulties by providing a teachers ensemble model that {\bf (1)} provides state-of-the-art accuracy, privacy, and fairness tradeoffs, {\bf (2)} adds negligible computational cost to standard (non-private) training, and {\bf (3)} directly exploits the existence of fair (non-private) algorithms rendering it practical even in the adoption of new fairness notions or when the fairness to be enforced would complicate the privacy analysis.

\section{Scalable, Fair, and Private Aggregation of Teacher Ensembles}
\label{sec:pate}

This section discusses two algorithms to transfer fairness constraints during the private learning process. Both algorithms rely on the presence of an ensemble of \emph{teacher} models $\bm{T} = \{\cM^{(i)}\}_{i=1}^K$, with each $\cM^{(i)}$  trained on a non-overlapping portion $D_i$ of the dataset $D$. This ensemble is used to transfer knowledge to a student model $\bar{\cM}_\theta : \cX \to \cY$. The student model $\bar{\cM}$ is trained using a dataset $\bar{D} \subseteq D_P$ whose samples are drawn i.i.d.~from
the same distribution $\Pi$ considered in Section \ref{sec:problem_setting} but whose protected group attributes are unrevealed. 
The framework, and the two variants introduced next, are depicted schematically in Figure \ref{fig:framework}.


\begin{figure}[t]
\centering
\includegraphics[width=0.8\linewidth]{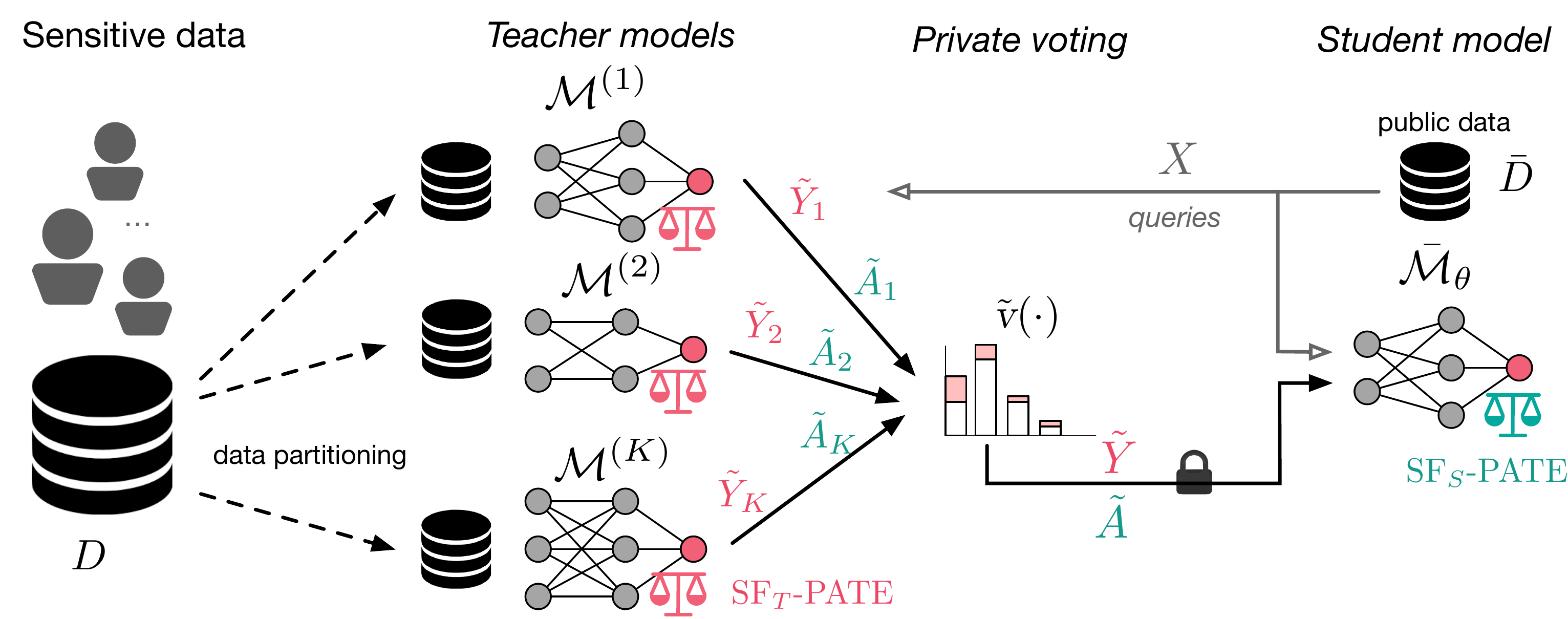}
\caption{Illustration of the SF-PATE framework. Green (red) colored text and labels depict the SF$_S$ (SF$_T$) version, in which the student (teachers ensemble) is trained under fairness constraints.}
\label{fig:framework}
\end{figure}

\subsection{$\text{SF}_S$-PATE: Semi-supervised Transfer Learning with Fair Student Model}
\label{sec:first_pate}

The first algorithm presented, called $\mbox{SF}_S$-PATE, trains a student model with privacy-preserving group attributes chosen by an ensemble of teachers. Subscript $_S$ in the algorithm's name stands for ``student'' to emphasize that is the student to enforce the fairness constraints during training, as the paper illustrates next. 

The teachers ensemble $\bm{T} = \{\cM^{(i)}\}_{i=1}^K$ is comprised of models $\cM^{(i)}\!:\!\cX \to \cA$, whose goal is to predict the group attributes (and not the labels, as classically done in classification tasks) given a sample's features.
Note that, importantly, the teacher models are standard classifiers (neither private nor fair). 
Their role is to transfer the information of the group attributes associated with the samples $(X_i, Y_i) \in \bar{D}$ provided by the student. We will see next how this information can be shared in a privacy-preserving manner. 

The student model $\bar{\cM}_\theta: \cX \to \cY$ is learned by minimizing the following regularized and constrained empirical risk function:
\begin{subequations}
\label{eq:sf_s}
    \begin{align}
    \label{eq:sf_s_A}
    \hat{\theta}   = \argmin_\theta & 
    \sum_{(X_i, Y_i) \in \bar{D}} \cL(\cM_{\theta}(X_i), Y_i) +\lambda \| \theta - \theta^* \|^2_2 \\
    \label{eq:sf_S_B}
    \text{s.t.}\quad
	  &
	  \xi \left( \{X_i, \tilde{v}_A\left( \bm{T}(X_i) \right), Y_i \}_{(X_i, Y_i) \in \bar{D}}, \theta \right) \leq \alpha,
	\end{align}
\end{subequations}
\noindent
where $\tilde{v}_A: \cA^K \to \cA$ is a privacy-preserving voting scheme returning the group attribute $\tilde{A}$ chosen by the teachers ensemble:
 \begin{equation}
  \tilde{A} = \tilde{v}_A(\bm{T}(X)) = \argmax_{a \in \cA } \{ \#_a(\bm{T}(X)) 
  + \mathcal{N}(0, \sigma^2) \},
  \label{eq:sf_s_pate_gnmax}
\end{equation}
which perturbs the reported counts $\#_a(\bm{T}(X))\!=\! \left\vert\{k \in [K] \mid \cM^{(k)}_{\theta}(X) = a\}\right\vert$ associated to group $a \in \cA$ with additive Gaussian noise of zero mean and standard deviation $\sigma$.

Problem \eqref{eq:sf_s} minimizes the standard empirical risk function (first component of \eqref{eq:sf_s_A}), while encouraging
the student model parameters $\theta$ to be close to the optimal, not fair, parameters $\theta^*$ (second component of \eqref{eq:sf_s_A}). 
Note that $\theta^*$ is obtained by simply solving problem \eqref{eq:erm} using dataset $\bar{D}$. The regularization parameter  $\lambda>0$ controls the trade-off between accuracy (large $\lambda$ values) and fairness (small $\lambda$ values). 
The fairness constraints expressed in Equation \eqref{eq:sf_S_B} can be enforced through the adoption of off-the-shelf techniques. In particular, this paper relies on the use of the Lagrangian Dual deep learning framework of Fioretto et al.~\cite{fioretto2020lagrangian} to enforce such constraints. 

\begin{wrapfigure}[10]{r}{0.43\textwidth}
\vspace{-20pt}
\includegraphics[width=\linewidth]{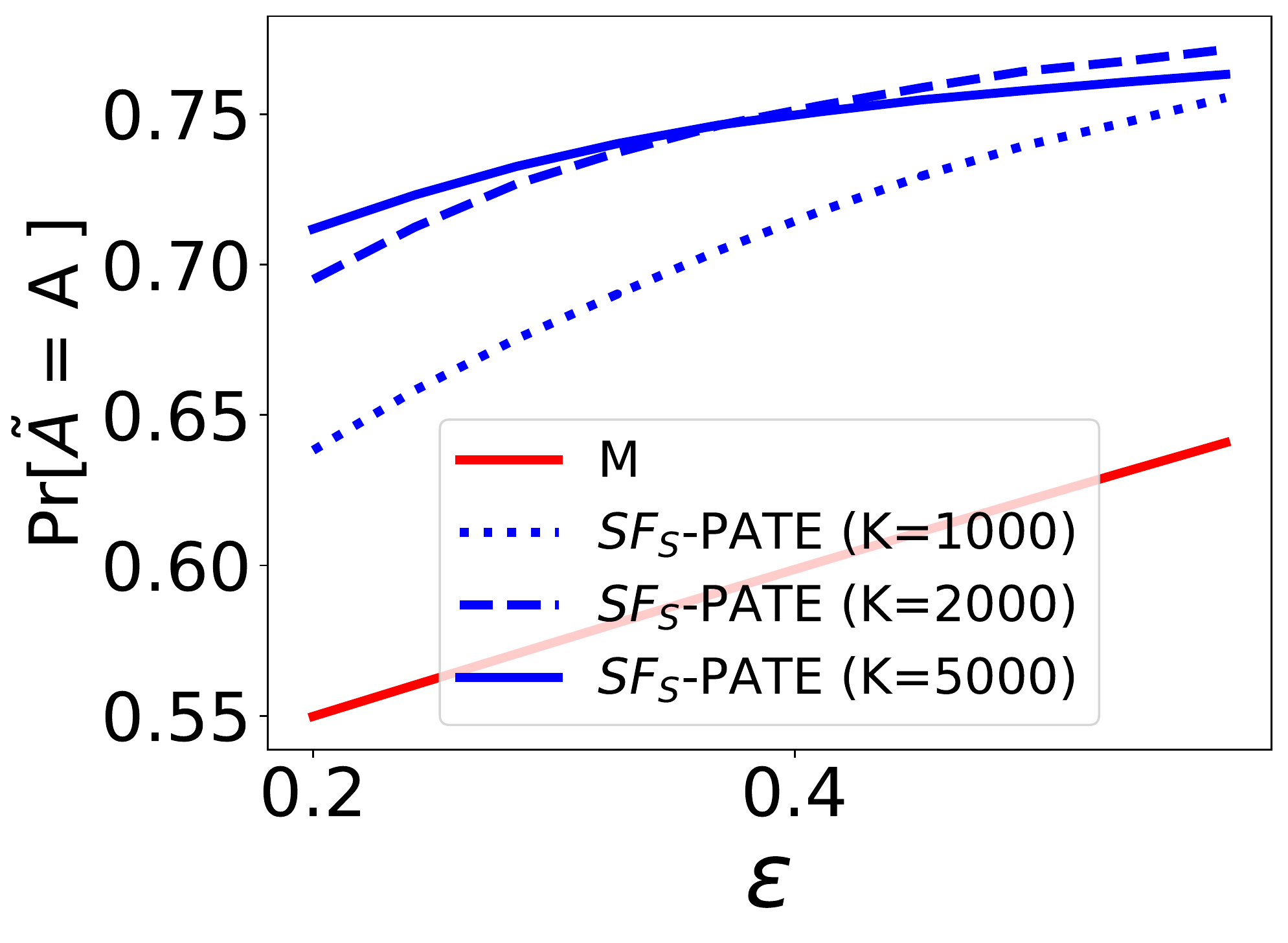}
\vspace{-25pt}
\caption{{\small Private group attributes accuracy.}}
\label{fig:accuracy_noisy_z}
\end{wrapfigure}
Figure \ref{fig:accuracy_noisy_z} illustrates the ability of the teachers ensemble to generate high-quality privacy-preserving group attributes: $\tilde{A} = \tilde{v}_A(\bm{T}(X))$, especially under tight privacy regimes. The figure compares the expected accuracy of the private group attributes $\pr{\tilde{A}_i = A_i}$ for different sizes $K$ of the ensemble on the Income dataset (See Appendix). It additionally illustrates the results obtained by a baseline model (\emph{M}) which adds calibrated noise to the group attributes using randomized response, which is the foundation of a state of the art method to train private and fair classifiers \cite{mozannar2020fair}. 


\smallskip
The rest of this section provides theoretical arguments to justify this surprising performance of $\mbox{SF}_S$-PATE, as well as shed lights on the fairness guarantees it can provide. 
The proofs of all theorems are reported in the Appendix. 
Note that $\mbox{SF}_S$-PATE relies on a (non-private) fair classifier to train the student model. While the performance of many such fair classifiers has been established in the literature, a natural question arises when adopting the proposed framework regarding the impact of the privacy-preserving voting scheme to fairness. 
The next theorem bounds the fairness violation of the student $\bar{\cM}_{\theta}$ w.r.t.~$(X,A,Y)$, when the group attribute $A$ and its privacy-preserving counterpart $\tilde{A}$ are close enough to each other statistically.

\begin{theorem}\label{thm:fair-bound-approach-1}
Let $\bar{\cM}_{\theta}$ be $\alpha$-fair w.r.t.~$(X,\tilde{A},Y)$ and $h(\cdot)$, and $\tilde{A}$ shares the same support $[m]$ as $A$ where, for $a\in [m], x\in \mathcal{X},
    y\in\mathcal{Y}$, the difference between the two conditional distributions $A$ and $\tilde{A}$ given
an event $(X=x, Y=y)$ is upper bounded by $\distdiff > 0$, i.e.,
$\left\vert \pr{\tilde{A} = a\mid X=x, Y=y}-\pr{A = a\mid X=x, Y=y}\right\vert \leq \distdiff$.
Then, $\bar{\cM}_{\theta}$ is $\alpha'$-fair w.r.t.~$(X,A,Y) $ and $h(\cdot)$ with:

\begin{equation}
    \alpha' = \frac{\distdiff\cdot \bound}{\underset{a\in[m]}{\min} \min\left\{\pr{\tilde{A}=a},\pr{A=a}\right\}}+\alpha\,.
\end{equation}
Furthermore, if the probability of $A$ belonging to any class $a\in [m]$ is strictly greater than $\distdiff$, i.e.,
\(
    \pr{A = a}>\distdiff,\,\forall~a\in[m],
\)
the fairness bound $\alpha'$ is given by
\begin{equation*}
    \alpha' = \frac{\distdiff\cdot \bound}{\underset{a\in[m]}{\min} \pr{A=a}-\distdiff}+\alpha\,.
\end{equation*}
\end{theorem}


The next result is associated with private group attribute $\tilde{A}$ generated by the aforementioned randomized response mechanism 
as a privacy-preserving step of the baseline model \emph{M} \cite{jagielski2019differentially} and also 
bounds fairness guarantees for the baseline model \emph{M}.
\begin{corollary}\label{cor:rr_mechanism}
    Suppose that $\bar{\cM}_{\theta}$ is $\alpha$-fair w.r.t. $(X,\tilde{A}, Y)$, along with $h(\cdot)$,
and $\tilde{A}$ is derived by the randomized response mechanism of $A$,
where, for any $a, a'\in[m]$,
\begin{equation*}
    \pr{\tilde{A}=a'\mid A=a}=\begin{cases}
                \frac{\exp(\epsilon)}{\exp(\epsilon)+(m-1)}\,,&\text{if } a'=a\,,\\
                \frac{1}{\exp(\epsilon)+(m-1)}\,,& \text{otherwise}\,.
        \end{cases}
\end{equation*}
Then, $\bar{\cM}_{\theta}$ is $\alpha'$-fair w.r.t. $(X,  A, Y) $ and $h(\cdot)$, where

\begin{equation*}
    \alpha' = \frac{\distdiff\cdot \bound}{\underset{a\in[m]}{\min} \min\{\pr{\tilde{A}=a},\pr{A=a}\}}+\alpha\,,\qquad\text{with }
    \distdiff=\frac{m-1}{\exp(\epsilon)+(m-1)}\,.
\end{equation*}
\end{corollary}

Corollary \ref{cor:rr_mechanism} provides a fairness guarantee of the student model
$\bar{\cM}_{\theta}$ w.r.t. the original group attribute $A$, when this model
is trained for achieving a given fairness violation. This is compared with the privacy-preserving counterpart $\tilde{A}$, which is
an output of the randomized response mechanism employed by the baseline model \emph{M}. 
This result sheds light on how fairness guarantees 
w.r.t. private group attributes can be transferred to those w.r.t. original group attributes.
Furthermore, when compared with the baseline model \emph{M}, $\mbox{SF}_S$-PATE makes the privacy-preserving group attributes $\tilde{A}$ much closer to their original counterparts $A$ under the same privacy budgets, as can be seen in Figure \ref{fig:accuracy_noisy_z}.
Finally, by Theorem \ref{thm:fair-bound-approach-1}, it can be inferred that SF-PATE is likely to achieve better fairness guarantees than those provided by models using randomized response, such as state-of-the-art model \emph{M}. 

\subsection{$\mbox{SF}_T$-PATE: Semi-supervised Transfer Learning with Fair Teachers Models}
\label{sec:second_pate}

While the previous algorithm focused on transferring the privacy-preserving group attributes from an ensemble of teachers to a fair student model, this section introduces an SF-PATE variant to transfer fairness {directly}, through the semi-supervised learning scheme. 
The proposed algorithm, called $\mbox{SF}_T$-PATE, trains a student model from an ensemble of fair teacher models. Subscript $_T$ in the algorithm's name stands for ``teachers'' to emphasize that the fairness constraints are enforced by the teachers and transferred to the student during training. 

The teachers ensemble  $\bm{T} = \{\cM^{(i)}\}_{i=1}^K$ is composed of pre-trained classifiers $\cM^{(i)}: \cX \to \cY$ that are non-private but fair over their training data $D_i$. Each SF$_T$-PATE teacher solves an empirical risk minimization problem subject to fairness constraint: 
\[
    \theta_i = \argmin_\theta \sum_{(X, Y) \in D_i} \cL(\cM_\theta(X_i), Y_i) \qquad \text{s.t.} \;\; 
    \xi(D_i, \theta) \leq \alpha.
\]
The implementation uses the Lagrangian dual method of \cite{fioretto2020lagrangian} to achieve this goal. This is again an important aspect of the SF-PATE framework since, by relying on black-box fair (and not private) algorithms, it decouples the dependency of developing joint private and fair analysis. 
The role of the teachers is to transfer fairness via model prediction to the student. The student model $\bar{\cM}_\theta : \cX \to \cY$ is learned by minimizing a standard empirical risk function:
\begin{equation}
    \min_{\theta} \sum_{X \in \bar{D} } 
    \cL\left( \cM_{\theta}(X), \tilde{v}_Y(\bm{T}(X)\right) 
    + \lambda \| \theta - \theta^*\|^2_2,
    \label{eq:main_sf_t_pate}
\end{equation}
where $\theta^*$ represents the student model parameters obtained solving the standard classification task (without fairness considerations) \eqref{eq:erm} on $\bar{D}$ and the private voting scheme $\tilde{v}_Y: \cY^K \to \cY$ selects the most voted target label $\tilde{Y}$ by the teachers ensemble: 
\begin{equation}
  \tilde{Y} = \tilde{v}(\bm{T}(X))  = \argmax_{y \in \cY} \{\#_y(\bm{T}(X) + \mathcal{N}(0, \sigma^2)\}.
  \label{eq:sf_t_pate_gnmax}
\end{equation}
Notice that, in addition to protecting the privacy of the sensitive group information $A$,  a biproduct of the above scheme is to also protect the labels $Y$ {when $\lambda\!=\!0$}. Thus, SF$_T$-PATE can be adopted in contexts where both the protected group and the labels are sensitive information. 
Finally, note that the voting scheme above, emulates the GNMAX framework proposed in \cite{papernot:16}, but the two have fundamentally different goals: GNMAX aims at protecting the participation of individuals into the training data, while SF$_T$-PATE aims at privately transferring the fairness properties of the teachers ensemble to a student model. 

\medskip 
Next, the paper sheds light on to which extent the fairness properties of the ensemble can be transferred to a student model, and provides theoretical insights to bound the fairness violations induced by the voting mechanism $\tilde{v}_Y$, adopted by SF$_T$-PATE. 
For notational convenience, let $B\coloneqq \sup_{y\in \mathcal{Y}, y'\in \mathcal{Y}}h(y, y')$ denote the supremum of the fairness function $h(\cdot)$.\footnote{One can consider $\bound=1$  for the fairness functions considered in Section \ref{sec:fairness}.}
Let $Z$ be the random vector $(\cM_{\theta}^{(1)}(X),\dots,\cM_{\theta}^{(K)}(X),Y)$ while $Z_a$ the conditional random
vector given the group label $a\in[m]$, i.e., 
$Z_a\coloneqq (\cM_{\theta}^{(1)}(X),\dots,\cM_{\theta}^{(K)}(X),Y)\mid A=a$.

\begin{definition}[Total variation distance]
    Let $Z, Z'$ be two probability distributions over a shared probability
space $\Omega$. Then, the total variation distance between them, denoted $\tv{Z}{Z'}$, is given by
\begin{align*}
    \tv{Z}{Z'} &\coloneqq \sup_{S\subseteqq \Omega}~\left\vert \pr{Z\in S}-\pr{Z'\in S}\right\vert=\sup_{S\subseteqq \Omega}~\left\vert\int_{S} \left(g_Z(\bm{z})-g_{Z'}(\bm{z})\right)d\bm{z}\right\vert\,,
\end{align*}
where $g_Z(\cdot)$ and $g_{Z'}(\cdot)$ are the probability density functions associated with $Z$ and $Z'$.
\end{definition}

The following results assume that, for each group attribute $a\in [m]$, the total variation distance between $Z$ and $Z_a$ is bounded from the above by $\distdiff$, i.e., $\tv{Z}{Z_a} \leq \distdiff$. 

The following theorem bounds the fairness violation of the voting mechanism $\tilde{v}_Y(\bm{T})$ when the joint distributions of the teachers ensemble and labels $Y$ are roughly similar across the different group classes. 
It shed light on the conditions required for the ensemble votes to transfer fairness knowledge accurately. 

\begin{theorem}
\label{thm:gen_ensemble}
The voting mechanism $\tilde{v}_Y(\bm{T})$ is  $\alpha'$ fair w.r.t. $(X,A,Y)$  and $h(\cdot)$ with 
\begin{equation*}
    \alpha' = \distdiff\cdot \bound\,.
\end{equation*}
\end{theorem}

\noindent
The next corollary is a direct consequence of Theorem \ref{thm:gen_ensemble} and provides a sufficient
condition for perfect fairness of the voting mechanism $\tilde{v}_Y(\bm{T})$.
\begin{corollary}
\label{thm:fair_ensemble}
Suppose that the random vector $Z$ is independent of $A$, i.e., 
$\{Z_a\}_{a\in[m]}$ and $Z$ are identically distributed.
Equivalently, the random vector $Z=(\cM_{\theta}^{(1)}(X),\dots,\cM_{\theta}^{(K)}(X),Y)$
is independent of $A$.
Then, the voting mechanism $\tilde{v}_Y(\bm{T})$ is perfectly fair w.r.t. $(X,A,Y)$  and $h(\cdot)$.
\end{corollary}
\noindent 

In summary, the results above provide sufficient conditions for the voting scheme of the ensemble to achieve a required fairness guarantees, which will be transferred to the subsequential student model.






 
 
 

\section{Privacy analysis}
\label{sec:privacy_computation}

The theoretical considerations above relate with the ability of the proposed models to preserve fairness. It remains to be shown that models are also private. 
The section first derives the sensitivity of the voting schemes adopted by $\mbox{SF}_S$- and $\mbox{SF}_T$-PATE. 
Given sensitive dataset $D_S$, define 
$\#_A(\bm{T}, D_S) = \langle \#_a(\bm{T}(X_i) \,|\, X_i \in D_S, a \in \cA \rangle$ and 
$\#_Y(\bm{T}, D_S) = \langle \#_y(\bm{T}(X_i) \,|\, X_i \in D_S, y \in \cY \rangle$, as the vectors of vote counts collected by the ensemble in $\mbox{SF}_S$- and $\mbox{SF}_T$-PATE, respectively.
\begin{theorem}
\label{thm:sensitivty_sf_s-pate}
 For any dataset $D_S$,
\(
\displaystyle \max_{D_S \sim D_S'}\| \#_{A}(\bm{T}, D_S)   - \#_{A}(\bm{T}, D_S') \|_2 = \sqrt{2}. 
\)

\end{theorem}


With  similar arguments, it can be shown that the sensitivity in reporting voting counts  in $\mbox{SF}_T$-PATE is also $\sqrt{2}$. 

\begin{theorem}
 For any dataset $D_S$, 
\(
\displaystyle
\max_{D_S \sim D_S'}\| \#_{Y}(\bm{T}, D_S)   - \#_{Y}(\bm{T}, D_S') \|_2 = \sqrt{2}. 
\)
\end{theorem}


Both proofs (detailed in the Appendix) rely on bounding the maximal difference of the count vectors $\#_A$ and $\#_Y$ arising across neighboring datasets.

\begin{proposition} 
\label{prop:3}
The voting scheme with private Gaussian noise $\mathcal{N}(0, \sigma^2)$  and sensitivity of $\sqrt{2}$ satisfies $(\gamma,\nicefrac{\gamma}{\sigma^2}) $-RDP for all $\gamma \geq 1$.
\end{proposition}

The proof of this proposition follows directly from \cite{mironov2017renyi}. Thus, for fixed standard deviation $\sigma$ and any $X \in \bar{D}$, the private voting schemes $\tilde{v}_{A}(\bm{T}(X))$ and $\tilde{v}_{Y}(\bm{T}(X))$ both satisfy $(\gamma,\nicefrac{\gamma}{\sigma^2}) $-RDP.
Note that the training uses $|\bar{D}| = s$ samples. The following results report composition bounds used to derive the final privacy cost.

\begin{theorem}[Composition for RDP]
\label{app:rdp_comp}
Let $\cM$ be a mechanism consisting of a sequence of adaptive mechanisms $\cM_1, \cM_2, \ldots, \cM_k$, such that for any $i \in [k]$, $\cM_i$ guarantees $(\gamma, \epsilon_i)$-RDP. Then $\cM$ guarantees $(\gamma, \sum_{i=1}^k \epsilon_i)$-RDP.
\end{theorem}

\begin{theorem}[From RDP to DP]
\label{thm:rdp_2_dp}
Let $\cM$ be an $(\gamma, \epsilon)$-RDP mechanism, then $\cM$ is also $(\epsilon +  \frac{\log \nicefrac{1}{\delta}}{\gamma-1}, \delta)$-DP for any $\delta \in (0,1)$.
\end{theorem}

\noindent
From Theorem \ref{app:rdp_comp}, Proposition \ref{prop:3} and Theorem \ref{thm:rdp_2_dp}, and for fixed noise parameter $\sigma$ used in the voting scheme, the SF-PATE algorithms satisfy $( \nicefrac{s\gamma}{\sigma^2} +\frac{\log \nicefrac{1}{\delta}}{\gamma -1}, \delta)$-DP.


\section{Experiments}
\label{sec:experiments}
This section evaluates the performance of the proposed SF-PATE algorithms against the prior approaches of \cite{tran2021differentially} and \cite{mozannar2020fair}, named {\em PF-LD} and {\em M}, respectively, and representing the current state-of-the-art for learning private and fair classifiers in the context studied in this paper.  
In addition to assess the competitiveness of the proposed solutions, the evaluation focuses on two key aspects: (1) It shows that SF-PATE 
can naturally handle new fairness notions, even when no viable privacy analysis exists about these notions. (2) It shows that SF-PATE has a low computational overhead 
compared to classical (non-private, non-fair) classifiers, rendering it a practical choice for the training of very large models. 
These two properties are unique to SF-PATE, enabling it to be applicable to a broad class of challenging decision tasks.

\smallskip\noindent{\textbf{Datasets and Settings}.} 
The evaluation is conducted using four UCI tabular datasets: Bank, Parkinson, Income and Credit Card \cite{UCIdatasets,article}, and UTKFace \cite{DBLP:conf/bmvc/HwangPKDB20}, a vision dataset. 
The latter is used to demonstrate the scalability of SF-PATE when trained on a ResNet 50 classifier with over 23 million parameters. All experiments are repeated 100 times using different random seeds.  

\smallskip\noindent{\textbf{Models and Hyperparameters}.}
To ensure a fair comparison, the experimental analysis uses the same architectures and parameters for all models (including the baselines models \emph{PF-LD} \cite{tran2021differentially} and \emph{M} \cite{jagielski2019differentially}). 
For tabular datasets, the underlying classifier is a feedforward neural network with two hidden layers and nonlinear ReLU activation. 
The fair, non-private, classifiers adopted by the two SF-PATE variants implement a Lagrangian Dual scheme \cite{fioretto2020lagrangian}, which is also the underlying scheme adopted by baseline models. 
For vision tasks on the UTK-Face dataset, the evaluation uses a Resnet 50 classifier. 

A more detailed description of these approaches, the settings adopted, and the datasets is deferred to the Appendix.

\vspace{-8pt}
\subsection{Accuracy, Privacy, and Fairness trade-off}

This section compares the accuracy, fairness, and privacy tradeoffs of the proposed models variants SF$_S$- and SF$_T$-PATE against the baseline models \emph{PF-LD} and \emph{M} on the tabular datasets. Figure \ref{fig:credit_card_compare} illustrates the accuracy (top subplots) and fairness violations $\xi(\theta, \bar{D})$ (bottom subplots) at varying of the privacy loss $\epsilon$ (x-axis). 
The figures clearly illustrate that both SF-PATE variants achieve better accuracy/fairness tradeoffs for various privacy parameters. 
The property of retaining high accuracy while low fairness violation is especially relevant in the tight privacy regime adopted ($\epsilon < 1$).
Additional results on other datasets are deferred to the Appendix, and show similar trends.

\begin{figure*}[tb]
\centering
\begin{subfigure}[b]{0.32\textwidth}
\includegraphics[width = 1.0\linewidth]{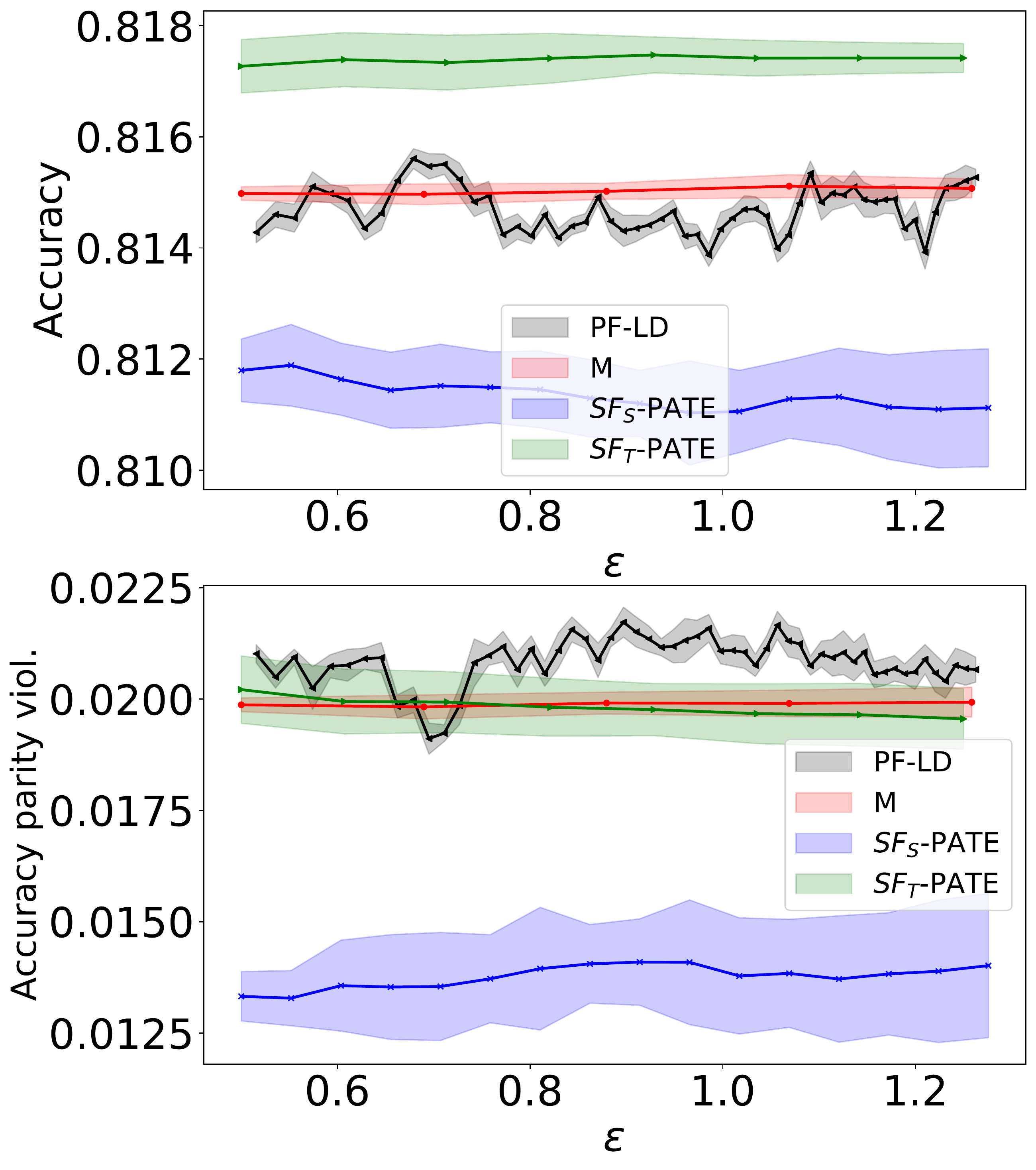}
\caption{Accuracy parity fairness}
\end{subfigure}
\begin{subfigure}[b]{0.32\textwidth}
\includegraphics[width = 1.0\linewidth]{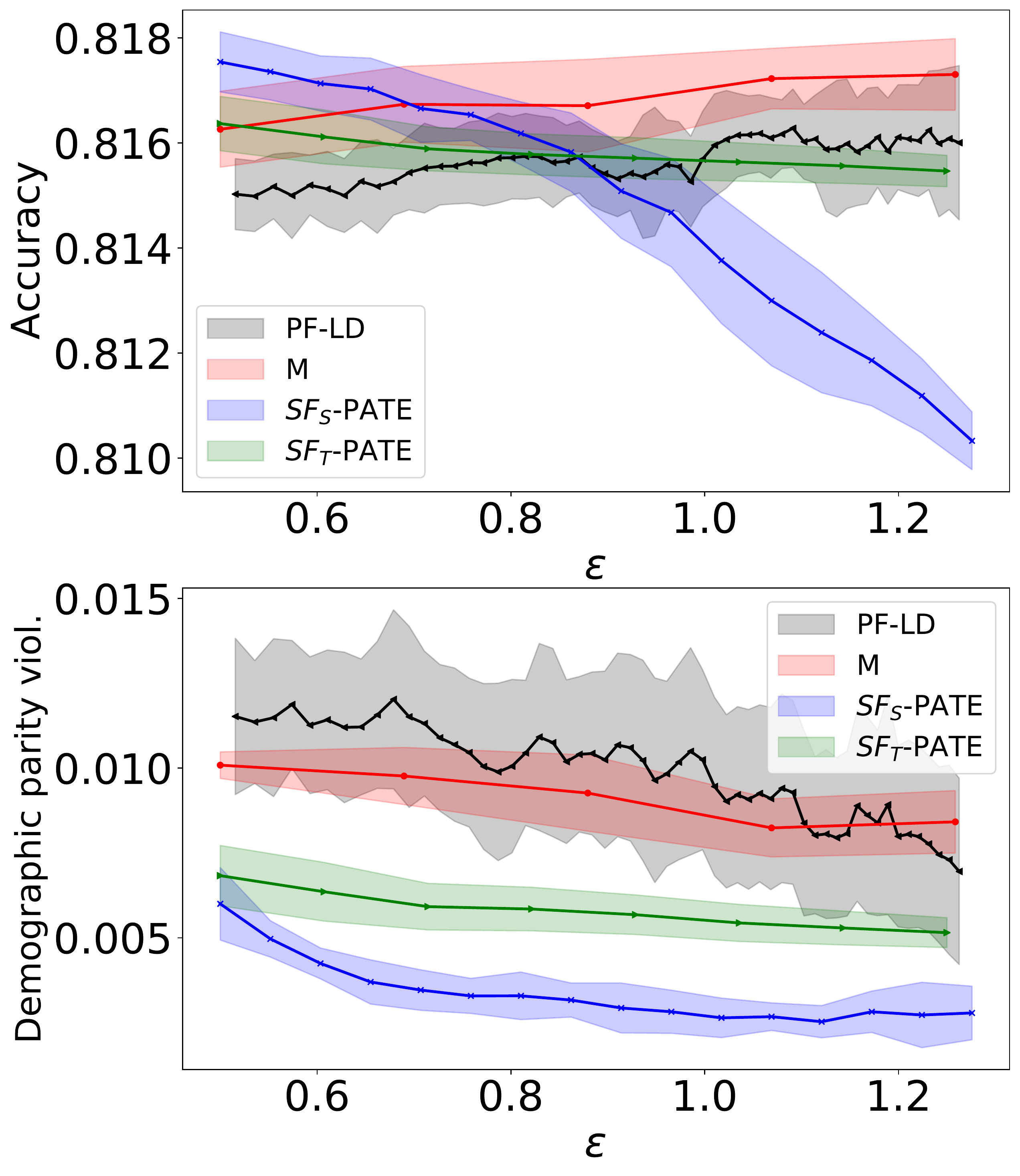}
\caption{Demographic parity fairness}
\end{subfigure}
\begin{subfigure}[b]{0.32\textwidth}
\includegraphics[width = 1.0\linewidth]{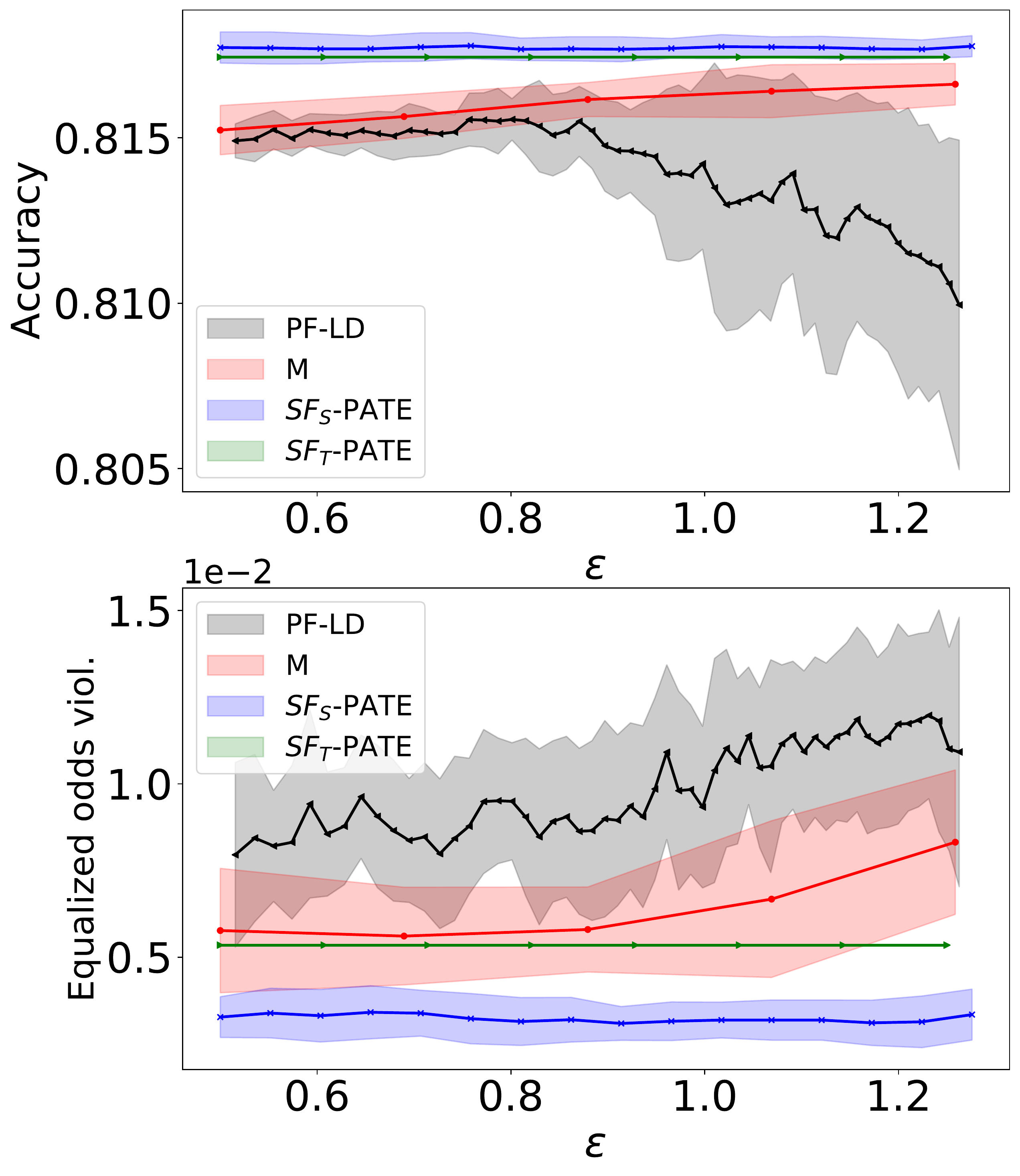}
\caption{Equalized odds fairness}
\end{subfigure}
\caption{Credit card data}
\label{fig:credit_card_compare}
\end{figure*}

First, notice that, consistently with the privacy-preserving ML literature showing that teachers ensemble models can outperform, in terms of accuracy, DP-SGD based models \cite{papernot:16,uniyal2021dp}, the SF-PATE models typically outperform the DP-SGD based \emph{FP-LD} model. Remarkably, both SF-PATE variants also report lower fairness violations, for the fairness notions analyzed, indicating the strength of this approach.
Additionally, recall that both $\mbox{SF}_S$-PATE and \emph{M} generate privacy-preserving group features and feed them to a fair model. However, in contrast to \emph{M}, the model ensemble used in $\mbox{SF}_S$-PATE exploits the relationship between the sample features $X$ and its associated group information $A$ to derive more accurate private group information $\tilde{A}$ (as already shown in Figure \ref{fig:accuracy_noisy_z}). This results in student models which are often more accurate than the baseline \emph{M}, while also being fairer. The second SF-PATE variant, $\mbox{SF}_T$-PATE, which operates by privately transferring the fairness knowledge from a teachers ensemble to a student, is found to always outperform both \emph{M} and \emph{PF-LD} on tabular datasets.
Finally, our analysis also shows that the average accuracy of the SF-PATE models is within 2\% of their non-private counterpart.  

 

\vspace{-8pt} 
\subsection{Handling new group fairness notions}

Next, the evaluation focuses on the ability of the SF-PATE framework to handle arbitrary fairness notions, even if a privacy analysis is missing, as long as a fair model can be derived. This is due to the framework property to rely on the use of black-box (non-private but fair) models. 
This is in sharp contrast with state-of-the-art model \emph{PF-LD}, that requires the development of a privacy analysis for each fairness property to be enforced, in order to calibrate the amount of noise to apply in both primal and dual step. 

To demonstrate this benefit, this section first introduces a new fairness notion, which generalizes demographic parity. 
\begin{definition}[Generalized demographic parity] 
It requires the distribution over the predictions of $\cM$ to be statistically independent of the protected group attribute A. That is, for all $a \in \cA$, and $\eta \in [0, 1]$
\(
    \Pr( \cM(X) \geq \eta | A = a ) = \Pr( \cM(X) \geq \eta ).
\)
\end{definition}
Note that the above is a general version of demographic parity, which states that $\pr{\cM_{\theta}(X) \geq 0.5 \mid A= a}  = \pr{\cM_{\theta}(X) \geq 0.5}$. 
Generalized demographic parity is useful in settings where the decision threshold (e.g, 0.5 above) might not be available at the training time. Matching the distribution of score functions (e.g., credit or income scores) among different groups attributes guarantees demographic parity fairness regardless of the decision threshold adopted. Such fairness constraint can be implemented by equalizing different order statistics of the score functions between group and population level:
 $$
    \EE[\cM_{\theta}(X)^h | A=a]  = \EE[\cM_{\theta}(X)^h] \quad \forall \ a \in [m],  h = 1, 2, \ldots, H 
$$
Practically, we let H = 2, and use the Lagrangian Dual method of \cite{fioretto2020lagrangian} to enforce these constraints during training. 
{\em Notice that it is highly non-trivial to derive a privacy analysis for such fairness notion}--the PF-LD model only does so for $H=1$, and, thus, not viable in this setting.

\begin{figure*}[tb]
    \centering
    \begin{subfigure}[b]{0.46\textwidth}
    \includegraphics[width = 1.0\linewidth]{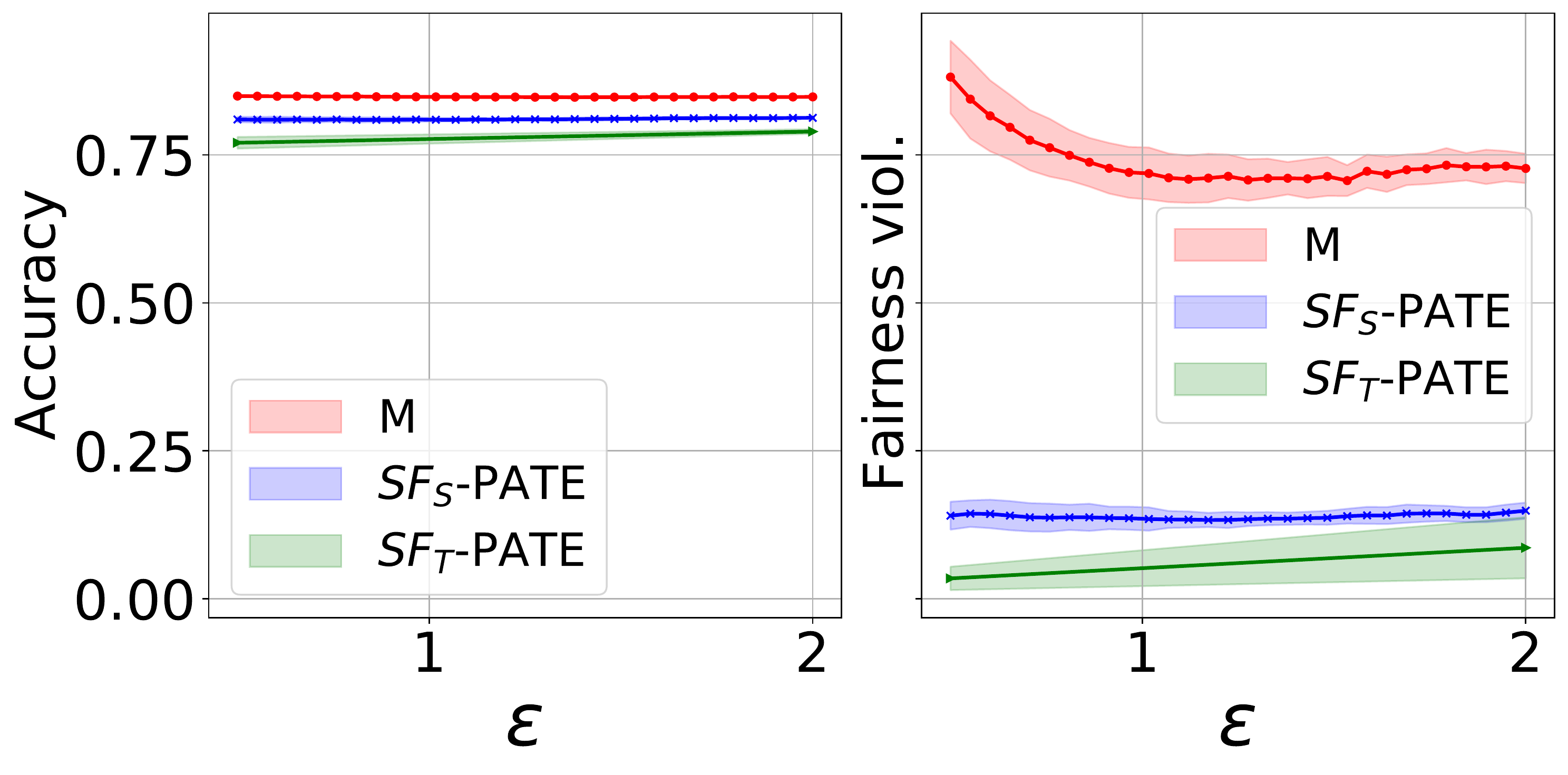}
    \caption{Income dataset}
    \end{subfigure}
    \begin{subfigure}[b]{0.46\textwidth}
    \includegraphics[width = 1.0\linewidth]{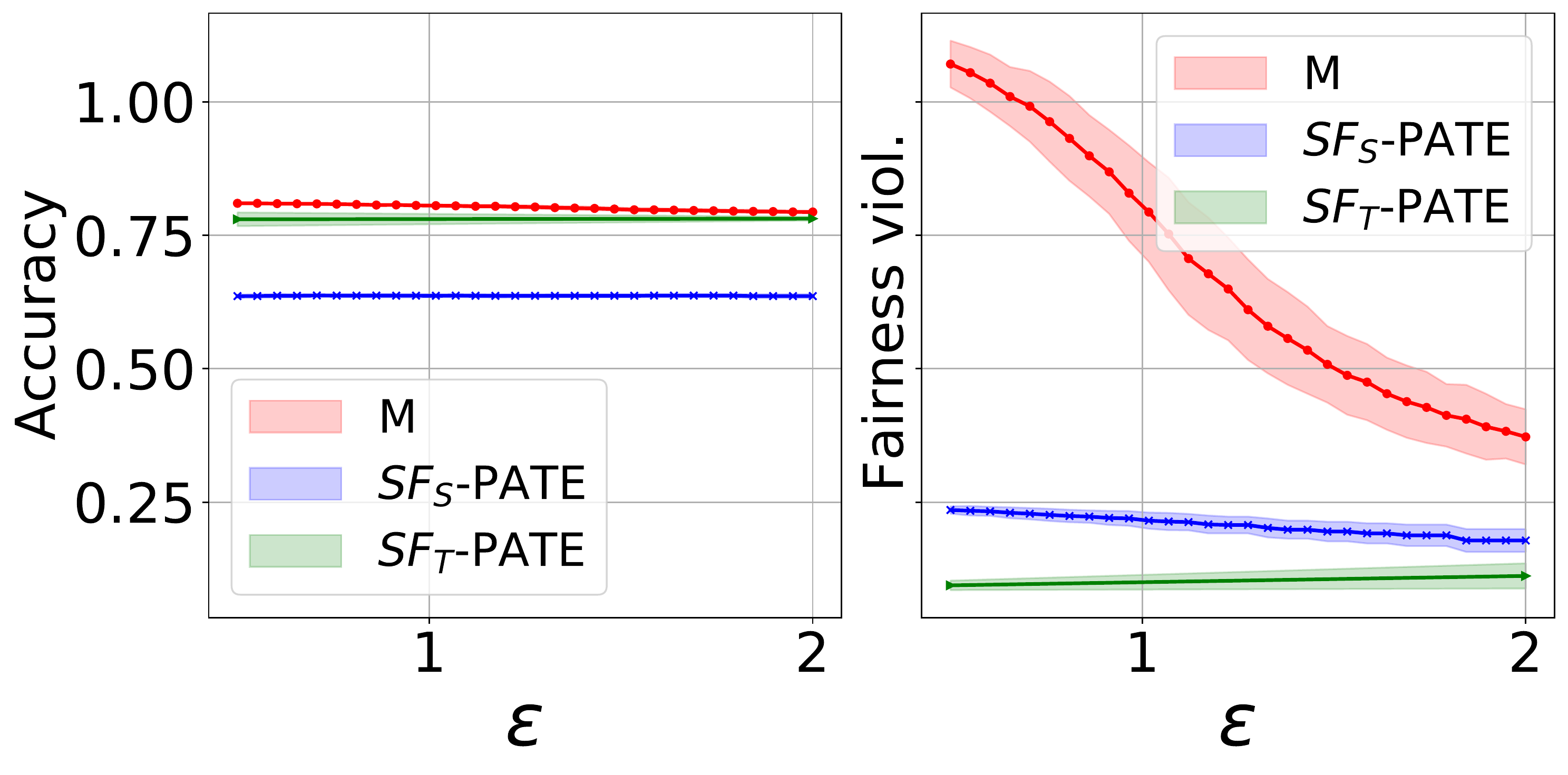}
    \caption{Bank dataset}
    \end{subfigure}
    \begin{subfigure}[b]{0.46\textwidth}
    \includegraphics[width = 1.0\linewidth]{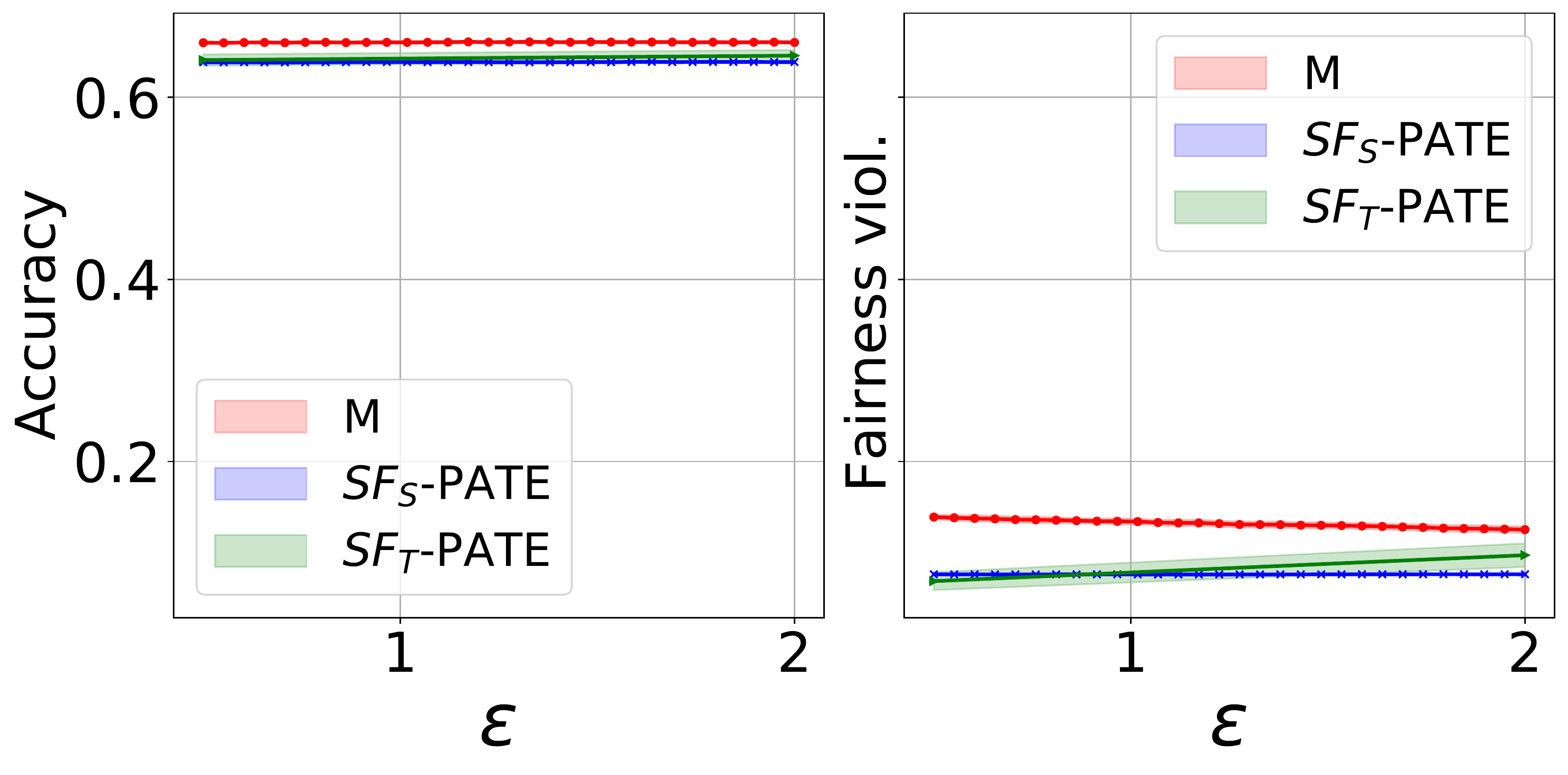}
    \caption{Parkinson dataset}
    \end{subfigure}
    \begin{subfigure}[b]{0.46\textwidth}
    \includegraphics[width = 1.0\linewidth]{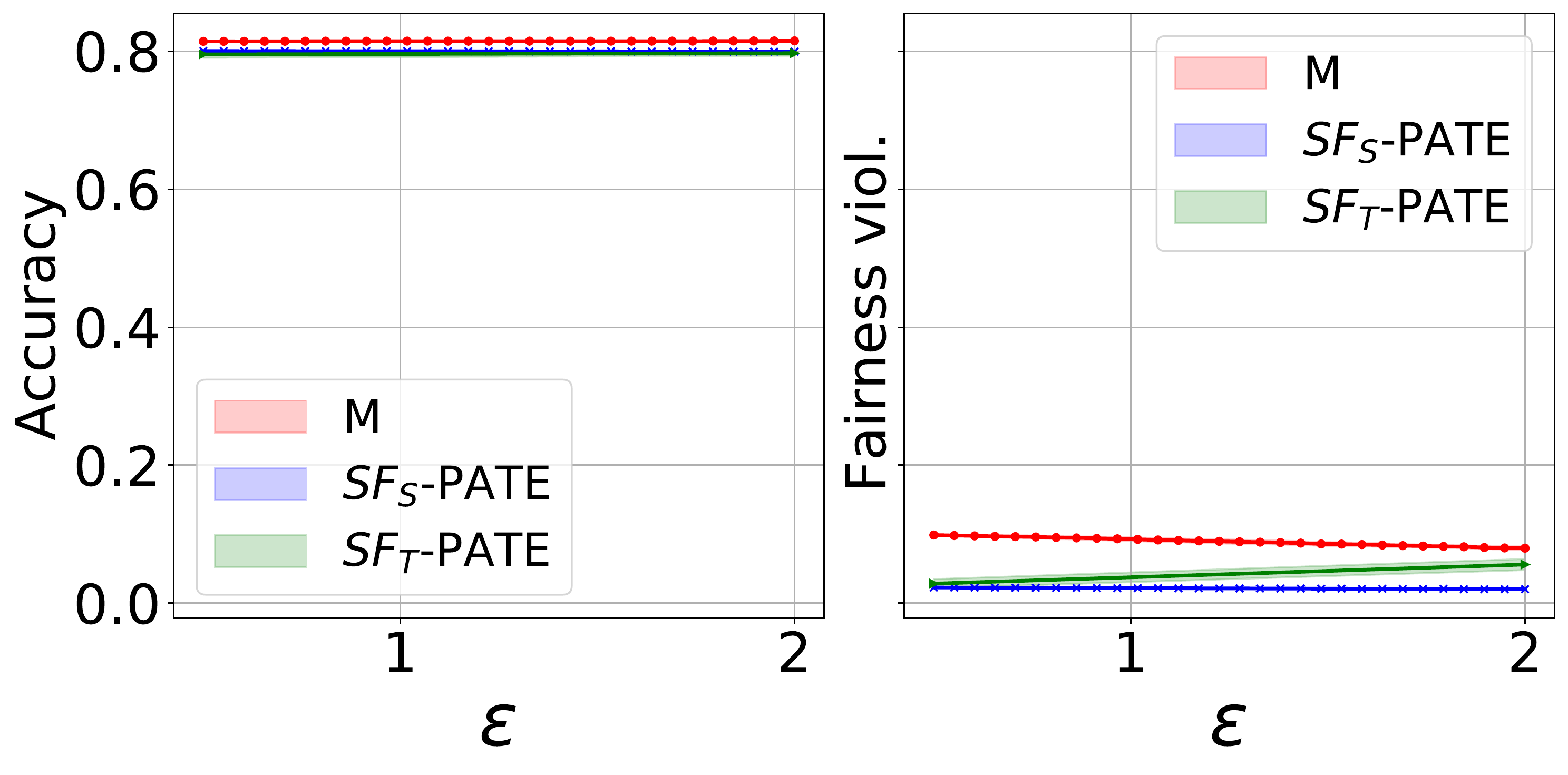}
    \caption{Credit card dataset}
    \end{subfigure}
    \caption{Comparison between our proposed methods with the baseline model M for generalized demographic parity fairness}
    \label{fig:generalized_dp}
\end{figure*}

Figure \ref{fig:generalized_dp} reports the accuracy and fairness violations obtained by the SF-PATE models and the baseline model \emph{M} which adopt, as a post-processing step, a non-private but fair classifier. 
Fairness is evaluated in terms of the Wasserstein distance between the score functions of different groups. The smaller the distance the lower the fairness violation.  The plots clearly illustrates the advantages of SF-PATE in terms of both accuracy and fairness when compared to model \emph{M}. Remarkably, the fairness violations reported by SF-PATE are often significantly lower than those reported by model \emph{M} for various privacy loss parameters $\epsilon$. 

{\em This is significant: The development of private and fair analysis is a non trivial task and framework like the proposed SF-PATE immediately lower the accessibility barrier for non-privacy expert users to develop both private and fair ML models.}

\subsection{Computational Time and Scalability}

\begin{figure}[t]
\centering
\begin{subfigure}[b]{0.8\textwidth}
\includegraphics[width = 1.0\linewidth]{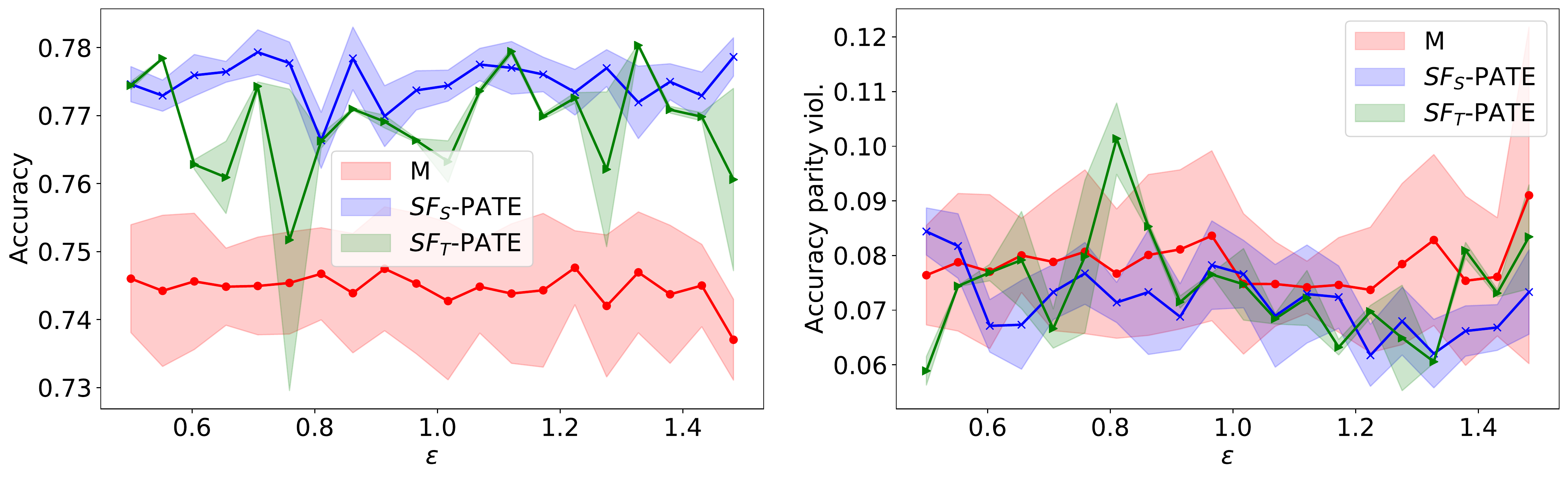}
\end{subfigure}
\caption{Accuracy/fairness of our models on UTKFace in which the base classifier is ResNet 50}
\label{fig:utk_resnet}
\end{figure}

This section demonstrates another key benefits of SF-PATE models: their ability to scale to large data and perform well on very deep networks. The experiments use a ResNet 50 ($>$ 23M parameters) to illustrate. In these experiments, PF-LD was unable to train even a single epoch ({in 1h}) due to the use of computational expensive per-sample gradient clipping performed during training. This renders such model unusable for many realistic settings. The comparison thus focuses on SF-PATE and \emph{M}. 
 
\begin{wrapfigure}[10]{r}{0.45\textwidth}
\vspace{-0pt}
\centering
\parbox{\linewidth}
{
\resizebox{\linewidth}{!}
{
    \begin{tabular}{r | rrrr} \toprule
        {\bf Dataset} & $ \mbox{SF}_S$-PATE &  $\mbox{SF}_T$-PATE
     & \emph{M} & PF-LD \\ \midrule
        Bank & 14 & 13 & 31 & 116  \\
    Parkinson & 8 & 8  & 17 & 31 \\
    Income & 55 & 56  & 129 & 1234 \\
    Credit Card & 30 & 31 & 76 & 575  \\
    UTK dataset & 1669  & 1662   &  3248  & N/A \\
    \bottomrule 
\end{tabular}
}
}
\vspace{-6pt}
\makeatletter\def\@captype{table}\makeatother
\caption{\small Runtime (in sec.) of different models to achieve $\epsilon= 1.0$ across different datasets. The ensemble size for $\mbox{SF}_S$-PATE and $\mbox{SF}_T$-PATE is $300$.} 
\label{tab:compare_time}
\end{wrapfigure}
Figure \ref{fig:utk_resnet} illustrates the accuracy and fairness tradeoff at varying of the privacy loss $\epsilon$. The models enforce accuracy parity (similar results can be obtained on other fairness notions). The figure illustrates how both SF-PATE variants bring dramatic accuracy improvements when compared to \emph{M}, at comparable or even lower fairness violations. 

Next, Table \ref{tab:compare_time} reports the training time required by the  algorithms compared to achieve a private model with parameter $\epsilon = 1.0$ on the benchmarks set analyzed. 
Notice the significant training time differences across the models with $\mbox{SF}_S$-PATE being up to three times faster than \emph{M} 
and up to two order of magnitude faster than \emph{PF-LD}. 

{\em These results illustrate the benefits of the proposed models, rendering it a practical tool for private and fair ML especially for large, overparametetrized models and under stringent privacy regimes.}

\vspace{-8pt}
\subsection{Discussion and Limitations}
While the previous section has shown the advantages of the proposed framework with respect to the state-of-the-art models, this section discusses some key aspects of SF$_S$- and SF$_T$-PATE and provides intuitions and guidelines for when one implementation may be preferred over the other. While both algorithms rely on the same framework, they differ in a fundamental aspect: the fairness enforcing mechanism. SF$_T$-PATE delegates fairness enforcement to its teachers while SF$_S$-PATE to its student. 
Notice that, by training an ensemble of fair teachers, SF$_T$-PATE is able to transfer fairness knowledge \emph{directly}, which, as observed in the experimental results, often results in smaller fairness violations than those attained by SF$_S$-PATE. This is notable, especially in the case of ``complex'' fairness constraints; i.e., when multiple concurrent constraints are to be enforced, as in the case of generalized demographic parity which imposes multiple order moments matching. 

A shortcoming of this algorithm, however, is that it has to enforce fairness in each teacher model. Thus, one has to ensure that enough data (with large enough representation from all the protected groups) is assigned to each teacher. This may be problematic in very scarce data regimes. On the other hand, by training a \emph{single} fair model (the student) SF$_S$-PATE avoids such potential issue. 

Finally, it is worth noting that a limitation of all ensemble models, including those proposed in this paper, is the need to store a model for each of the $K$ teachers. This can be problematic when the models become very large and additional considerations should be taken in such context. This, however, also represents an opportunity for future research to develop effective model storage and model pruning techniques that minimize the loss in accuracy and fairness while retaining privacy.

\section{Conclusions}

This paper was motivated by the need of building models whose outcomes do not discriminate against some demographic groups while also requiring privacy protection of these sensitive attributes.  The paper proposed a framework to train deep learning models that satisfy several notions of group fairness while ensuring that the model satisfies differential privacy for the protected attributes. 
The proposed framework, called \emph{Scalable, Fair, and Private Aggregation of Teacher Enseble} (SF-PATE) transfer fairness knowledge learned from a pretrained ensemble of models to a target model via a privacy-preserving voting process. 
The paper introduced two flavors of SF-PATE. The first, called SF$_T$-PATE, allows transferring properties of the ensemble while protecting the demographic group attribute, and the second, called SF$_S$-PATE, in addition, also protects the target labels. 
An important properties of SF-PATE is to allow the adoption of black-box (non-private) fair models during the knowledge transfer process, which ultimately, we believe, will simplify the development and boosts the adoption of new fairness metrics in privacy-preserving ML. 

The paper also analyzes the fairness properties of SF-PATE and shows that unfairness can be bounded in many practical settings.
Finally, evaluation on both tabular and image datasets shows not only that SF-PATE achieves better accuracy, privacy, and fairness tradeoffs with respect to the current state-of-the-art, but it is also significantly faster. These properties render SF-PATE amenable to train large, overparameterized, models, that ensure privacy, accuracy, and fairness simultaneously, showing that it may become a practical tool for privacy-preserving and fair decision making.


\section*{Acknowledgement}
Cuong Tran and Keyu Zhu are to be considered both first authors.  
This research is partially supported by NSF grant 2133169 and a Google Research Scholar award. Its views and conclusions are those of the authors only.

\bibliographystyle{splncs04}
\bibliography{ijcai22}

\appendix
\include{appendix}

\end{document}

%% file: inputs.tex
\usepackage{scalerel}
\usepackage{url} 
\usepackage{booktabs}       
\usepackage{nicefrac}       
\usepackage{subcaption}
\usepackage{amsmath}
\usepackage{amssymb}
\usepackage{amsfonts}
\usepackage{tikz}
\usepackage{pgfplots}
\usepackage{amsfonts}
\usepackage{booktabs}
\usepackage{siunitx}
\pgfplotsset{compat=1.7}
\usepgfplotslibrary{fillbetween}
\usepackage{fontawesome}

\usepackage[ruled,linesnumbered,noresetcount,vlined]{algorithm2e}
\makeatletter

\newcommand{\nosemic}{\renewcommand{\@endalgocfline}{\relax}}
\newcommand{\dosemic}{\renewcommand{\@endalgocfline}{\algocf@endline}}
\let\oldnl\nl
\newcommand{\nonl}{\renewcommand{\nl}{\let\nl\oldnl}}

\SetKwFor{Repeat}{repeat}{:}{endw}
\makeatother
\usepackage{todonotes}
\usepackage{xcolor}
\usepackage{setspace}
\usepackage{multirow}
\usepackage{bm,bbm}
\usepackage{paralist}
\usepackage{enumitem}

\DeclareMathOperator*{\argmax}{argmax}
\DeclareMathOperator*{\argmin}{argmin}

\definecolor{darkgreen}{RGB}{204,102,0}

\newcommand{\cA}{\mathcal{A}}

\newcommand{\cL}{\mathcal{L}}
\newcommand{\cM}{\mathcal{M}} \newcommand{\cN}{\mathcal{N}}

 \newcommand{\cY}{\mathcal{Y}}
 \newcommand{\cX}{\mathcal{X}}

\newcommand{\EE}{\mathbb{E}} \newcommand{\RR}{\mathbb{R}}

\newcommand{\pr}[1]{\operatorname{Pr}\left(#1\right)}
\usepackage{mathtools}
\usepackage{esvect}

\newcommand{\tv}[2]{\operatorname{d_{TV}}(#1,#2)}

\newcommand{\distdiff}{\eta}
\newcommand{\bound}{B}

\pgfmathdeclarefunction{laplace}{2}{%
  \pgfmathparse{-1/(2*#2)*exp(-abs(x-#1)/#2)}%
}
\definecolor{maincolor}{HTML}{032F99}
\definecolor{blue}{RGB}{31,64,122}
\definecolor{red}{HTML}{e05a87} 
\makeatletter
\newcommand{\oset}[3][0ex]{%
  \mathrel{\mathop{#3}\limits^{
    \vbox to#1{\kern-2\ex@
    \hbox{$\scriptstyle#2$}\vss}}}}
\makeatother
\newcommand{\optimal}[1]{\oset{\scalebox{.5}{$\star$}}{#1}}

\usepackage{letltxmacro}
\LetLtxMacro\orgvdots\vdots
\LetLtxMacro\orgddots\ddots

\makeatletter
\DeclareRobustCommand\vdots{%
  \mathpalette\@vdots{}%
}
\newcommand*{\@vdots}[2]{%
  \sbox0{$#1\cdotp\cdotp\cdotp\m@th$}%
  \sbox2{$#1.\m@th$}%
  \vbox{%
    \dimen@=\wd0 %
    \advance\dimen@ -3\ht2 %
    \kern.5\dimen@
    \dimen@=\wd2 %
    \advance\dimen@ -\ht2 %
    \dimen2=\wd0 %
    \advance\dimen2 -\dimen@
    \vbox to \dimen2{%
      \offinterlineskip
      \copy2 \vfill\copy2 \vfill\copy2 %
    }%
  }%
}
\DeclareRobustCommand\ddots{%
  \mathinner{%
    \mathpalette\@ddots{}%
    \mkern\thinmuskip
  }%
}
\newcommand*{\@ddots}[2]{%
  \sbox0{$#1\cdotp\cdotp\cdotp\m@th$}%
  \sbox2{$#1.\m@th$}%
  \vbox{%
    \dimen@=\wd0 %
    \advance\dimen@ -3\ht2 %
    \kern.5\dimen@
    \dimen@=\wd2 %
    \advance\dimen@ -\ht2 %
    \dimen2=\wd0 %
    \advance\dimen2 -\dimen@
    \vbox to \dimen2{%
      \offinterlineskip
      \hbox{$#1\mathpunct{.}\m@th$}%
      \vfill
      \hbox{$#1\mathpunct{\kern\wd2}\mathpunct{.}\m@th$}%
      \vfill
      \hbox{$#1\mathpunct{\kern\wd2}\mathpunct{\kern\wd2}\mathpunct{.}\https://www.overleaf.com/project/5efc88f13fedba0001692ea8m@th$}%
    }%
  }%
}
\makeatother

\pgfplotsset{
    mark repeat*/.style={
        scatter,
        scatter src=x,
        scatter/@pre marker code/.code={
            \pgfmathtruncatemacro\usemark{
                or(mod(\coordindex,#1)==0, (\coordindex==(\numcoords-1))
            }
            \ifnum\usemark=0
                \pgfplotsset{mark=none}
            \fi
        },
        scatter/@post marker code/.code={}
    }
}

\makeatletter
\newcommand*{\rom}[1]{\expandafter\@slowromancap\romannumeral #1@}
\makeatother

%% file: appendix.tex
\onecolumn

\setcounter{theorem}{0}
\setcounter{lemma}{0}
\setcounter{proposition}{0}


\section{Missing Proofs in Section \ref{sec:pate}}

\begin{theorem}
Let $\bar{\cM}_{\theta}$ be $\alpha$-fair w.r.t.~$(X,\tilde{A},Y)$ and $h(\cdot)$, and $\tilde{A}$ shares the same support $[m]$ as $A$ where, for $a\in [m], x\in \mathcal{X},
    y\in\mathcal{Y}$, the difference between the two conditional distributions $A$ and $\tilde{A}$ given
an event $(X=x, Y=x)$ is upper bounded by $\distdiff > 0$, i.e.,
\begin{equation}\label{eq:t5-assumption-1}
\left\vert \pr{\tilde{A} = a\mid X=x, Y=y}-\pr{A = a\mid X=x, Y=y}\right\vert \leq \distdiff.
\end{equation}
Then, $\bar{\cM}_{\theta}$ is $\alpha'$-fair w.r.t.~$(X,A,Y) $ and $h(\cdot)$ with:

\begin{equation}
    \alpha' = \frac{\distdiff\cdot \bound}{\underset{a\in[m]}{\min} \min\left\{\pr{\tilde{A}=a},\pr{A=a}\right\}}+\alpha\,.
\end{equation}
Furthermore, if the probability of $A$ belonging to any class $a\in [m]$ is strictly greater than $\distdiff$, i.e.,
\(
    \pr{A = a}>\distdiff,\,\forall~a\in[m],
\)
the fairness bound $\alpha'$ can then be given by
\begin{equation*}
    \alpha' = \frac{\distdiff\cdot \bound}{\underset{a\in[m]}{\min} \pr{A=a}-\distdiff}+\alpha\,.
\end{equation*}
\end{theorem}

\begin{proof}
Note that, for any $a\in [m]$,
\begin{align}
    &\left\vert \pr{\tilde{A}=a}-\pr{A=a}\right\vert\nonumber\\
    =~&\left\vert \int g_{X,Y}(x, y)\cdot\left[\pr{\tilde{A} = a\mid X=x, Y=y}-\pr{A = a\mid X=x, Y=y}\right]dxdy \right\vert\label{eq:t5-1}\\
    \leq~&  \int g_{X,Y}(x, y)\cdot\left\vert \pr{\tilde{A} = a|X=x, Y=y}-\pr{A = a\mid X=x, Y=y}\right\vert dxdy\nonumber\\
    \leq~&\int g_{X,Y}(x, y)\cdot \distdiff~dxdy\label{eq:t5-2}\\
    =~&\distdiff\nonumber\,,
\end{align}
where the function $g_{X,Y}(\cdot, \cdot)$ in Equation \eqref{eq:t5-1} is the probability density function associated with the joint distribution
$(X,Y)$ and Equation \eqref{eq:t5-2} is due to the assumption in Equation \eqref{eq:t5-assumption-1}. If the probability that the random vector $A$ equals $a$ is
    strictly greater than $\distdiff$ for any $a\in [m]$, it follows that
\begin{equation}\label{eq:t5-a-tilde-a}
    \pr{\tilde{A}=a}\geq \pr{A=a}-\distdiff >0\,, \qquad\forall~a\in[m]\,.
\end{equation}
    For any group attribute $a\in [m]$, 
    \begin{align}
        &\left\vert\mathbb{E}_{X,Y|A=a } [ h(\bar{\cM}_{\theta}(X),Y)]  -  \mathbb{E}_{X,Y} 
        [ h(\bar{\cM}_{\theta}(X),Y)]\right\vert\nonumber\\
        \leq~& \left\vert\mathbb{E}_{X,Y|A=a } [ h(\bar{\cM}_{\theta}(X),Y)]  -  \mathbb{E}_{X,Y|\tilde{A}=a } [ h(\bar{\cM}_{\theta}(X),Y)]\right\vert +\nonumber\\
        &\left\vert\mathbb{E}_{X,Y|\tilde{A}=a } [ h(\bar{\cM}_{\theta}(X),Y)]  -  \mathbb{E}_{X,Y} [ h(\bar{\cM}_{\theta}(X),Y)]\right\vert\nonumber\\
        \leq~&\left\vert\mathbb{E}_{X,Y|A=a } [ h(\bar{\cM}_{\theta}(X),Y)]  -  \mathbb{E}_{X,Y|\tilde{A}=a } [ h(\bar{\cM}_{\theta}(X),Y)]\right\vert+\alpha\label{eq:t5-main-1}\\
        =~&\left\vert\int h(\bar{\cM}_{\theta}(x), y)\cdot \left(g_{(X,Y)\mid A=a}(x, y)-g_{(X,Y)\mid \tilde{A}=a}(x, y)\right)dx dy\right\vert+\alpha\nonumber\\
        \leq~& \int h(\bar{\cM}_{\theta}(x), y)\cdot g_{X,Y}(x, y) \left\vert\frac{\pr{A = a\mid X=x, Y=y}}{\pr{A = a}}-
        \frac{\pr{\tilde{A} = a\mid X=x, Y=y}}{\pr{\tilde{A} = a}}\right\vert dx dy+\alpha\nonumber\\
        \leq ~& \int h(\bar{\cM}_{\theta}(x), y)\cdot g_{X,Y}(x, y) \frac{\distdiff}{\min\left\{\pr{A = a},\pr{\tilde{A} = a}\right\}}dx dy+\alpha\label{eq:t5-main-3}\\
        =~&\frac{\distdiff\cdot \mathbb{E}_{X,Y} [ h(\bar{\cM}_{\theta}(X),Y)]}{\min\left\{\pr{A = a},\pr{\tilde{A} = a}\right\}}+\alpha\,,\nonumber
    \end{align}
    where Equation \eqref{eq:t5-main-1} is due to the assumption that $\cM_{\theta}$ is $\alpha$-fair w.r.t. $(X, \tilde{A},Y)$ and $h(\cdot)$. 
    Equation \eqref{eq:t5-main-3} comes from non-negativity of the functions $h(\cdot)$ and $g(\cdot)$ and the assumption in Equation
    \eqref{eq:t5-assumption-1}.
    Putting things together, we have the following.
    \begin{align*}
        &\max_{a\in [m]}~\left\vert\mathbb{E}_{X,Y|A=a } [ h(\bar{\cM}_{\theta}(X),Y)]-\mathbb{E}_{X,Y } [ h(\bar{\cM}_{\theta}(X),Y)]\right\vert\\
        \leq ~&\max_{a\in [m]}~\frac{\distdiff\cdot \mathbb{E}_{X,Y}
        [ h(\bar{\cM}_{\theta}(X),Y)]}{\min\left\{\pr{A = a},\pr{\tilde{A} = a}\right\}}+\alpha=\frac{\distdiff\cdot \mathbb{E}_{X,Y} [ h(\bar{\cM}_{\theta}(X),Y)]}{\underset{a\in[m]}{\min} \min\left\{\pr{A = a},\pr{\tilde{A} = a}\right\}}+\alpha\\
        \leq ~&\frac{\distdiff\cdot \bound}{\underset{a\in[m]}{\min} \min\left\{\pr{A = a},\pr{\tilde{A} = a}\right\}}+\alpha\,,
    \end{align*}
    where the last inequality comes from the assumption that $\bound$ is the supremum of the fairness function $h(\cdot)$. 
    Suppose that the probability that the random vector $A$ equals $i$ is
    strictly greater than $\distdiff$, for any group attribute $a\in [m]$. It follows that
\begin{align*}
    &\max_{a\in [m]}~\left\vert\mathbb{E}_{X,Y|A=a } [ h(\bar{\cM}_{\theta}(X),Y)]-\mathbb{E}_{X,Y} [ h(\bar{\cM}_{\theta}(X),Y)]\right\vert\\
    \leq ~&\frac{\distdiff\cdot \bound}{\underset{a\in[m]}{\min} \min\left\{\pr{A = a},\pr{\tilde{A} = a}\right\}}+\alpha\\
    \leq~& \frac{\distdiff\cdot \bound}{\underset{a\in[m]}{\min} \pr{A=a}-\distdiff}+\alpha\,,
\end{align*}
where the last inequality comes from Equation \eqref{eq:t5-a-tilde-a}.
\end{proof}

\begin{corollary}
    Suppose that $\bar{\cM}_{\theta}$ is $\alpha$-fair w.r.t. $(X,\tilde{A}, Y)$, along with $h(\cdot)$,
and $\tilde{A}$ is derived by the randomized response mechanism of $A$,
where, for any $a, a'\in[m]$,
\begin{equation*}
    \pr{\tilde{A}=a'\mid A=a}=\begin{cases}
                \frac{\exp(\epsilon)}{\exp(\epsilon)+(m-1)}\,,&\text{if } a'=a\,,\\
                \frac{1}{\exp(\epsilon)+(m-1)}\,,& \text{otherwise}\,.
        \end{cases}
\end{equation*}
Then, $\bar{\cM}_{\theta}$ is $\alpha'$-fair w.r.t. $(X,  A, Y) $ and $h(\cdot)$, where

\begin{equation*}
    \alpha' = \frac{\distdiff\cdot \bound}{\underset{a\in[m]}{\min} \min\{\pr{\tilde{A}=a},\pr{A=a}\}}+\alpha\,,\qquad\text{with }
    \distdiff=\frac{m-1}{\exp(\epsilon)+(m-1)}\,.
\end{equation*}
\end{corollary}

\begin{proof}
    In order to derive the argument in this corollary, it suffices to show the following: for any 
    group attribute $a\in [m]$, $x\in \cX$ and $y\in \cY$,
    \begin{equation*}
        \left\vert \pr{\tilde{A} = a\mid X=x, Y=y}-\pr{A = a\mid X=x, Y=y}\right\vert 
        \leq \frac{m-1}{\exp(\epsilon)+(m-1)}\,.
    \end{equation*}
    It follows that, for any group attribute $a\in[m]$, $x\in\mathcal{X}$ and $y\in\mathcal{Y}$,
    \begin{align*}
        &\left\vert \pr{\tilde{A} = a\mid X=x, Y=y}-\pr{A = a\mid X=x, Y=y}\right\vert\\
        =~& \left\vert\sum_{a'=1}^m \pr{\tilde{A} = a\mid A=a'}\cdot \pr{A = a'\mid X=x, Y=y}-
        \pr{A = a\mid X=x, Y=y}\right\vert\\
        =~&\left\vert\frac{-(m-1)}{\exp(\epsilon)+(m-1)}\cdot \pr{A = a\mid X=x, Y=y}
        +\frac{\pr{A \neq a\mid X=x, Y=y}}{\exp(\epsilon)+(m-1)}\right\vert\\
        =~&\left\vert\frac{1-m\cdot \pr{A = a\mid X=x, Y=y}}{\exp(\epsilon)+(m-1)}\right\vert\\
        \leq ~&\frac{m-1}{\exp(\epsilon)+(m-1)}\,.
    \end{align*}
\end{proof}

\begin{theorem}
The voting mechanism $\tilde{v}_Y(\bm{T})$ is  $\alpha'$ fair w.r.t. $(X,A,Y)$  and $h(\cdot)$ with 
\begin{equation*}
    \alpha' = \distdiff\cdot \bound\,.
\end{equation*}
\end{theorem}

\begin{proof}
    Consider an arbitrary function $q: \mathbb{R}^K\times \mathbb{R}^{\vert\cY\vert}\mapsto\mathbb{R}$ of the outputs of the $K$ teacher networks.
    For notational convenience, let $w_q$ denote the function associated with $q$ such that
    \begin{equation*}
        w_q(\cM_{\theta}^{(1)}(X),\dots,\cM_{\theta}^{(K)}(X),Y,\tau)\coloneqq h(q(\cM_{\theta}^{(1)}(X),\dots,\cM_{\theta}^{(K)}(X),\tau),Y)\,,
    \end{equation*}
    where $\tau$ follows a multivariate normal distribution $\cN(0, \sigma^2 I_{\vert\cY\vert})$.
    It follows that, for any group attribute $a\in [m]$,
    \begin{multline*}
        \mathbb{E}_{Z_a,\tau}[w_q(Z_a,\tau)]-\mathbb{E}_{Z,\tau}[w_q(Z,\tau)]=\\
        \mathbb{E}_{\tau,X,Y|A=a } [ h(q(\cM_{\theta}^{(1)}(X),\dots,\cM_{\theta}^{(K)}(X),\tau),Y)]  -  \mathbb{E}_{\tau,X,Y} [ h(q(\cM_{\theta}^{(1)}(X),\dots,\cM_{\theta}^{(K)}(X),\tau),Y)]\,,
    \end{multline*}
    and
    \begin{align}
        &\left\vert\mathbb{E}_{Z_a,\tau}[w_q(Z_a,\tau)]-\mathbb{E}_{Z,\tau}[w_q(Z,\tau)]\right\vert\nonumber\\
        =~&\left\vert\int w_q(\bm{z},\tau)\cdot
        \left(g_{Z_a}(\bm{z})-g_Z(\bm{z})\right) d\bm{z}d\tau\right\vert\nonumber\\
        =~&\left\vert\int \left(\int w_q(\bm{z},\tau)d\tau\right)\cdot
        \left(g_{Z_a}(\bm{z})-g_Z(\bm{z})\right) d\bm{z}\right\vert\nonumber\\
        =~& \left\vert\underbrace{\int_{S_a}  r(\bm{z})\cdot
        \left(g_{Z_a}(\bm{z})-g_Z(\bm{z})\right) d\bm{z}}_{\text{(I)}}+\underbrace{\int_{\RR^{K}\setminus S_a} r(\bm{z})\cdot
        \left(g_{Z_a}(\bm{z})-g_Z(\bm{z})\right) d\bm{z}}_{\text{(II)}}\right\vert\label{eq:t4-main-1}\\
        \leq ~& \max\left\{
         \int_{S_a}  r(\bm{z})\cdot
        \left(g_{Z_a}(\bm{z})-g_Z(\bm{z})\right) d\bm{z},
        ~-\int_{\RR^{K}\setminus S_a} r(\bm{z})\cdot
        \left(g_{Z_a}(\bm{z})-g_Z(\bm{z})\right) d\bm{z}
        \right\}\label{eq:t4-main-1_5}\\
        \leq~& \max\left\{B \int_{S_a}
        g_{Z_a}(\bm{z})-g_Z(\bm{z}) d\bm{z},~
        B\int_{\RR^{K}\setminus S_a}
        g_{Z}(\bm{z})-g_{Z_a}(\bm{z}) d\bm{z}\right\}\label{eq:t4-main-2}\\
        \leq ~& B\cdot \distdiff \label{eq:t4-main-3}\,,
    \end{align}
    where $r(\bm{z})$ is a shorthand for $\int w_q(\bm{z},\tau)d\tau$
    and $S_a$ in Equation \eqref{eq:t4-main-1} is defined as $\{\bm{z}\in\RR^{K+1}\mid g_{Z_a}(\bm{z})
    \geq g_{Z}(\bm{z})\}$. Equation \eqref{eq:t4-main-1_5} comes from the fact that the term (I) in Equation \eqref{eq:t4-main-1} 
    is non-negative while
    the term (II) is non-positive. Equation \eqref{eq:t4-main-2} comes from the fact
    that $\bound$ is the supremum of the fairness function $h$ (and equivalently $w_q$). Equation \eqref{eq:t4-main-3}
    is due to the assumption that the total variation distance between $Z_a$ and $Z$ is upper bounded by $\distdiff$.
    Therefore, 
    \begin{align*}
         &\max_{a\in[m]}~\Large\vert \mathbb{E}_{\tau,X,Y\mid A=a } [ h(q(\cM_{\theta}^{(1)}(X),\dots,\cM_{\theta}^{(K)}(X),\tau),Y)]  \\ 
         & \hspace{20pt} -  \mathbb{E}_{\tau,X,Y} [ h(q(\cM_{\theta}^{(1)}(X),\dots,\cM_{\theta}^{(K)}(X),\tau),Y)]
         \Large\vert\\
        =~&\max_{a\in[m]}~\left\vert\mathbb{E}_{Z_a,\tau}[w_q(Z_a,\tau)]-\mathbb{E}_{Z,\tau}[w_q(Z,\tau)]\right\vert\\
        \leq~& B\cdot \distdiff\,,
    \end{align*}
    where the aggregated classifier $q(\cM_{\theta}^{(1)}(X),\dots,\cM_{\theta}^{(K)}(X), \tau)$ associated with
    an arbitrary operation $q$ is $(B\cdot \distdiff)$-fair w.r.t.
    $(X,A,Y)$ and $h(\cdot)$. 
    Note that, if $q$ is in the following form:
    \begin{multline*}
        q(\cM_{\theta}^{(1)}(X),\dots,\cM_{\theta}^{(K)}(X), \tau)\coloneqq \\
         \underset{Y \in \mathcal{Y} }{\arg\max}~\#_Y(\cM_{\theta}^{(1)}(X),\dots,\cM_{\theta}^{(K)}(X)) + \tau_Y=\tilde{v}_Y(\bm{T}(X))\,,
    \end{multline*}
    we end up with the voting mechanism $\tilde{v}_Y(\bm{T}(X))$.
    Therefore,
    the voting mechanism $\tilde{v}_Y(\bm{T}(X))$ is $(B\cdot \distdiff)$-fair as well.
    
\end{proof}

\begin{corollary}
Suppose that the random vector $Z$ is independent of $A$, i.e., 
$\{Z_a\}_{a\in[m]}$ and $Z$ are identically distributed.
Equivalently, the random vector $Z=(\cM_{\theta}^{(1)}(X),\dots,\cM_{\theta}^{(K)}(X),Y)$
is independent of $A$.
Then, the voting mechanism $\tilde{v}_Y(\bm{T})$ is perfectly fair w.r.t. $(X,A,Y)$  and $h(\cdot)$.
\end{corollary}

\begin{proof}
    Notice that, for any $a\in[m]$, $Z$ now shares the same distribution with $Z_a$, which indicates that the total
    variation distance between $Z$ and $Z_a$, $\tv{Z}{Z_a}$, is $0$. Therefore,
    the voting mechanism $\tilde{v}_Y(\bm{T}(X))$ is $0$-fair 
    (and equivalently perfectly fair) w.r.t. $(X,A,Y)$  and $h(\cdot)$.
\end{proof}

\section{Missing Proofs in Section \ref{sec:privacy_computation}}

\begin{theorem}
 For any dataset $D_S$,
\(
\displaystyle \max_{D_S \sim D_S'}\| \#_{A}(\bm{T}, D_S)   - \#_{A}(\bm{T}, D_S') \|_2 = \sqrt{2}. 
\)

\end{theorem}

\begin{proof}
Since $D_S$ differs with $D_S'$ by at most one single sample and the dataset $D_S$ is partitioned into disjoint
parts, that  sample can only belong to a single part of the partition, say $D_i$. As a consequence,
the teacher classifier associated with $D_i$, which is $\cM^{(i)}(X)$, 
can change its prediction from $a \in \cA$ to $a' \in \cA$.  
Hence, $ \#_a(\bm{T}, D_S)$ decreases by 1 while $ \#_{a'}(\bm{T}, D_S')$ increases by 1  when switching from $D_S$ to $D_S'$.  To sum up, the two  vectors of voting counts $ \#_A(\bm{T}(X), D_S) $ and  $\#_A(\bm{T}(X), D_S') $ differ by at most two entries, which indicates that the $\ell_2$ sensitivity of reporting voting counts is $\sqrt{2}$.
\end{proof}

With  similar arguments, we can show that the sensitivity in reporting voting counts  in $\mbox{SF}_T$-PATE is also $\sqrt{2}$. 

\begin{theorem}
 For any dataset $D_S$, 
\(
\displaystyle
\max_{D_S \sim D_S'}\| \#_{Y}(\bm{T}, D_S)   - \#_{Y}(\bm{T}, D_S') \|_2 = \sqrt{2}. 
\)
\end{theorem}

\begin{proof}
When we change the group attributes of one sample $X \in D$, this can affect up to one fair teacher, say
$\cM^{(i)}$. Like Theorem \ref{thm:sensitivty_sf_s-pate}, in the worst case,
the fair classifier  $\cM^{(i)}$ can change its prediction from $y \in \cY$ to $y' \in \cY$. 
Hence, the vectors of voting counts of $\#_{Y}(\bm{T}, D_S) $ and $\#_{Y}(\bm{T}, D_S') $ differ by at most two entries. Therefore, the $\ell_2$ sensitivity of reporting counts in $\mbox{SF}_T$-PATE is $\sqrt{2}$. 
\end{proof}


\section{Experimental Results: Additional Details}
We provide additional details of experimental settings, datasets,  and empirical results in this section. 

\subsection{Experimental settings}

\smallskip\noindent\textbf{Neural network architectures }
For tabular datasets, the underlying classifiers for all models are  feed-forward networks with two hidden layers and nonlinear ReLU activation function. We use the same network architecture for teacher models in both of our approaches.

For vision tasks on UTK-Face, the base classifier in to-be deployed models (student model in PATE-based methods) will be a Resnet 50  with 23M parameters. To speed up training time, we will use a simple convolutional network which  consists of two convolutional layers followed by three fully connected linear layer for teacher model in both   $\mbox{SF}_S$-PATE and $\mbox{SF}_T$-PATE. Another reason is due to the recent work \cite{cho2019efficacy} which shows student performance might degrade unexpectedly when distilled from oversized teachers. Hence, simpler models are preferred for teacher models.

\smallskip\noindent\textbf{Fair solver settings } 

For a fair comparison, all models used in experiments are implemented based on a fair solver which relies on Dual Lagrangian framework \cite{fioretto2020lagrangian}. We  set the fairness parameter $\alpha = 0$ for maximum fairness,  and train the fair solver in 200 epochs, with learning rate of 1e-3, batch size of 32, and the step-sizes to update multipliers is 1e-2.

\smallskip\noindent\textbf{Hyper-parameters settings }
For both of  our proposed  models,  we set $s$ which is  the number of samples in $\bar{D}$ for four tabular datasets and UTK Face to be 200 and 500 respectively. We set the trade-off parameters  $\lambda = 1e-3$. We optimize the ensemble size $K$ by using the public evaluation set.

For PF-LD, we optimize the clipping bounds in primal and dual step by using the public evaluation set. We set the noise multipliers in primal and dual step to be 5 and 50 respectively \cite{tran2021differentially}.

\smallskip\noindent\textbf{Computing Infrastructure} 
All of our experiments on tabular datasets are performed on a personal desktop machine with 16GB RAM and Core I7.

For experiments on UTKFace and ResNet 50 classifier, we ran experiments on a cluster of  300 GPUs with NVIDIA RTX A6000 graphics cards.

\smallskip\noindent\textbf{Software and Libraries}  
All models and experiments were written in Python 3.7. All neural network classifier models in our paper were implemented in Pytorch 1.5.0. 

The Tensorflow Privacy package was also employed for computing the privacy loss.  We also set privacy parameter $\delta= 1e-4$ in all of our experiments.  


\smallskip\noindent\textbf{Datasets}  
Our paper evaluates all models based on the following four UCI datasets: \emph{Bank}, \emph{Income}, \emph{Parkinsons}, \emph{Credit card} and UTK Face dataset. 
A descriptions of each dataset is reported as follows:

\begin{enumerate}
    \item \textbf{Income} (Adult) dataset which has 45K samples of 52 features. The task is to predict if an 
    individual has low or high income, and the group attributes are defined by race: 
    \emph{White} vs \emph{Non-White} \cite{UCIdatasets}.

    \item \textbf{Bank} dataset which has 11K samples of 17 features. The  classification task is to predict if a user subscribes 
    a term deposit or not and the group attributes are defined by age: \emph{people whose age is less than vs greater than 60 years old} \cite{UCIdatasets}. 
    
    \item \textbf{Parkinsons} dataset which has 6K samples of 18 features . The classification task is to predict if 
    a patient has total UPDRS score that exceeds the median value, 
    and the group attributes are defined by gender: \emph{female vs male} \cite{article}.  
       
    \item \textbf{Credit Card} dataset which has 30K samples of 21 features.  The task is to predict if 
    a customer defaults a loan or not. The group attributes are defined by gender:
    \emph{female vs male} \cite{creditdataset}. 
    
    \item \textbf{UTKFace} dataset which has more than 20K facial images of 48x48 pixels resolution.  The task is to predict one of five predefined races of a given facial image. The group attributes are defined based on the following 9 age ranges: 0-10, 10-15, 15-20, 20-25, 25-30, 30-40, 40-50, 50-60, 60-120.  \cite{DBLP:conf/bmvc/HwangPKDB20}.
\end{enumerate}

On each dataset we perform standardization to render all input features with zero mean and unit standard deviation. Each dataset was partitioned into three disjoint subsets: train (70\%), evaluation (5\%)  and test set(25\%). We use the evaluation set as an available public dataset  to optimize the hyper-parameters such as clipping bounds in PF-LD or ensemble size $K$ in our proposed frameworks for a fair comparison.

\subsection{Accuracy, Privacy, and Fairness trade-off}

We provide additional results regarding the accuracy, fairness, and privacy tradeoffs of proposed models variants SF$_S$ and SF$_T$-PATE against the baseline models PF-LD and M on Bank, Income and Parkinson dataset here.  . Figure \ref{fig:generalized_dp} illustrates the accuracy (left subplots) and fairness violations $\xi(\theta, \bar{D})$ (right subplots) at varying of the privacy loss $\epsilon$ (x-axis). These figures again demonstrate that both SF-PATE variants achieve better accuracy/fairness tradeoffs for various privacy parameters than baseline models.

\begin{figure*}
\centering
\begin{subfigure}[b]{0.32\textwidth}
\includegraphics[width = 1.0\linewidth]{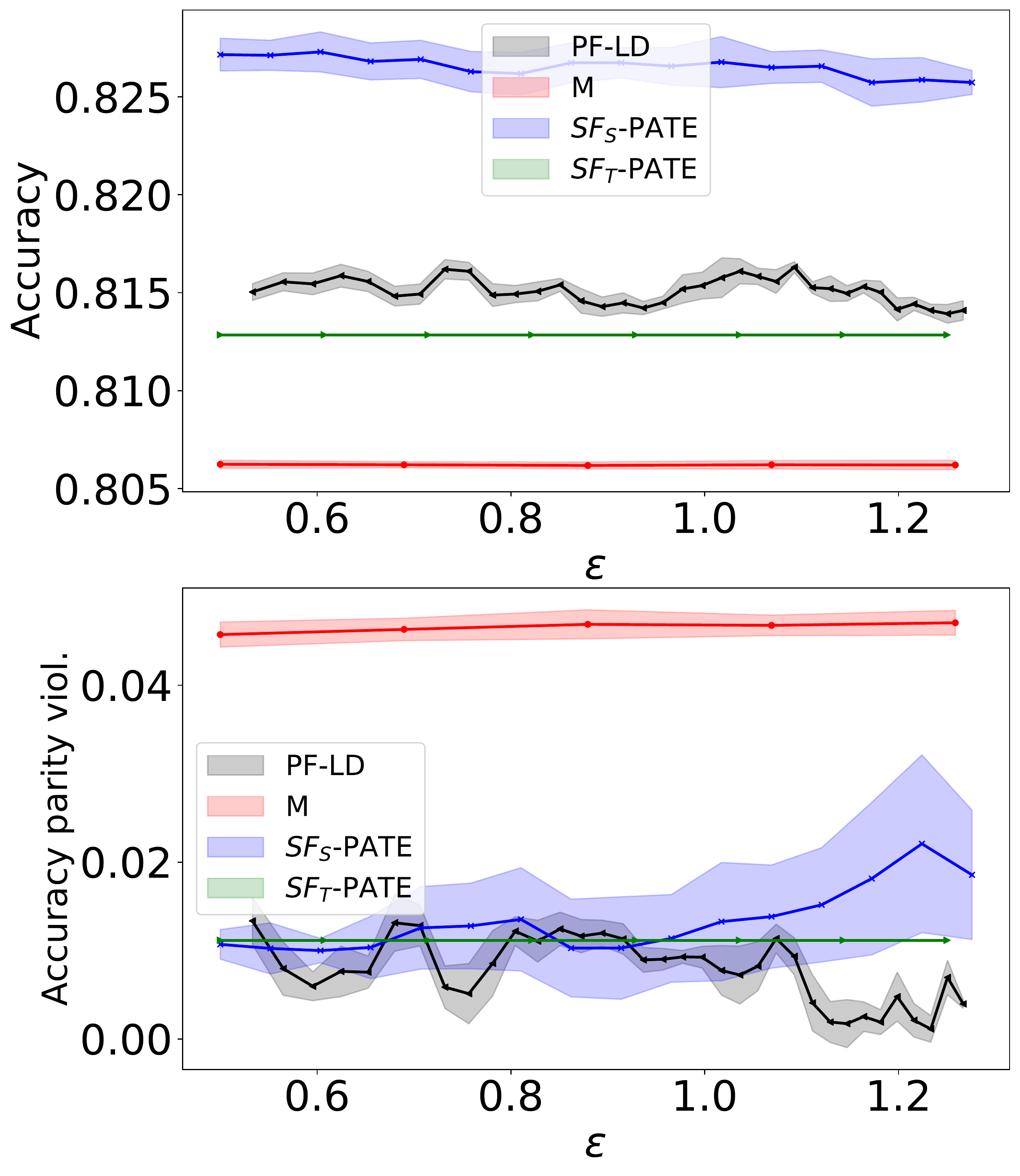}
\caption{Accuracy parity fairness}
\end{subfigure}
\begin{subfigure}[b]{0.32\textwidth}
\includegraphics[width = 1.0\linewidth]{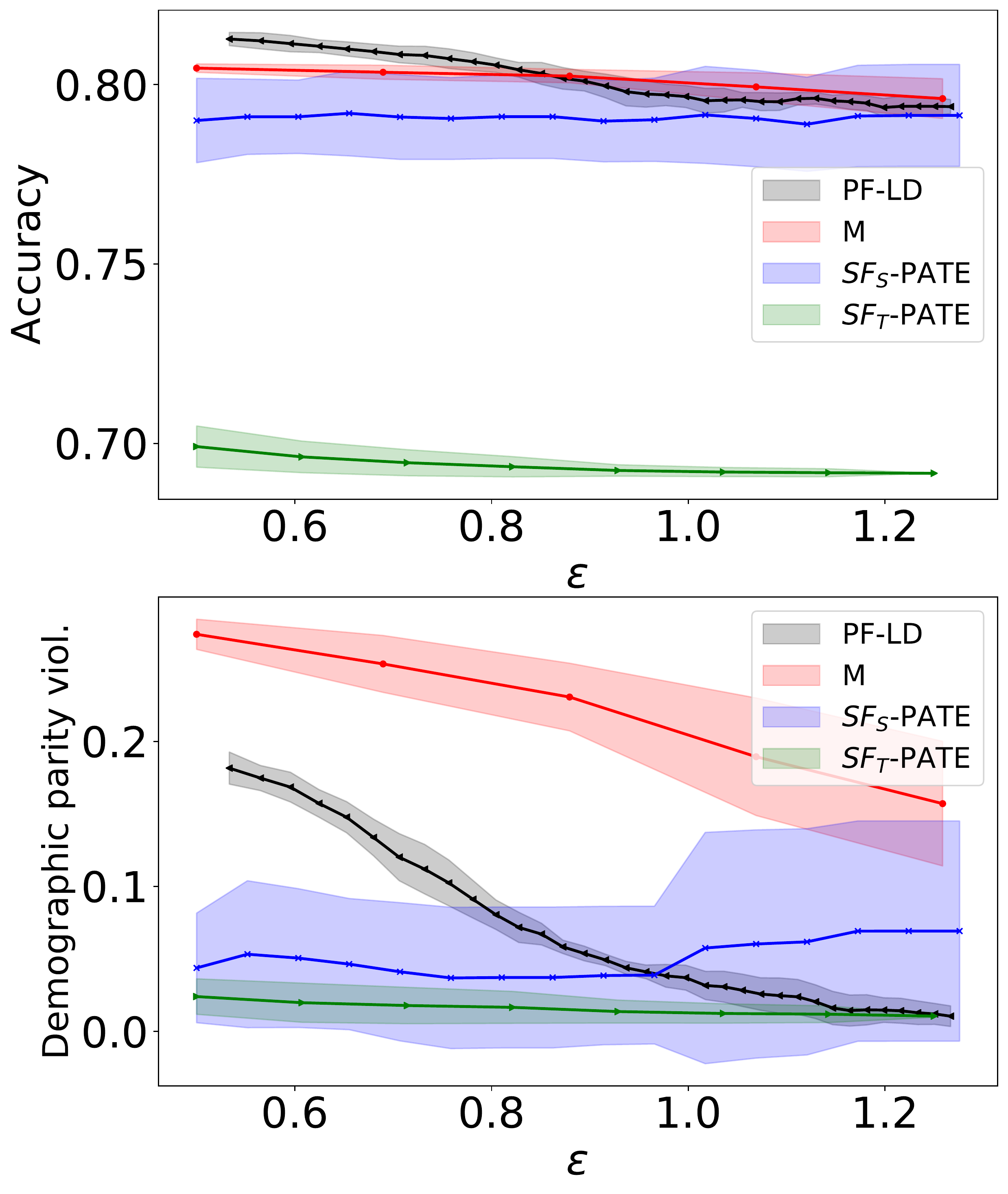}
\caption{Demographic parity fairness}
\end{subfigure}
\begin{subfigure}[b]{0.32\textwidth}
\includegraphics[width = 1.0\linewidth]{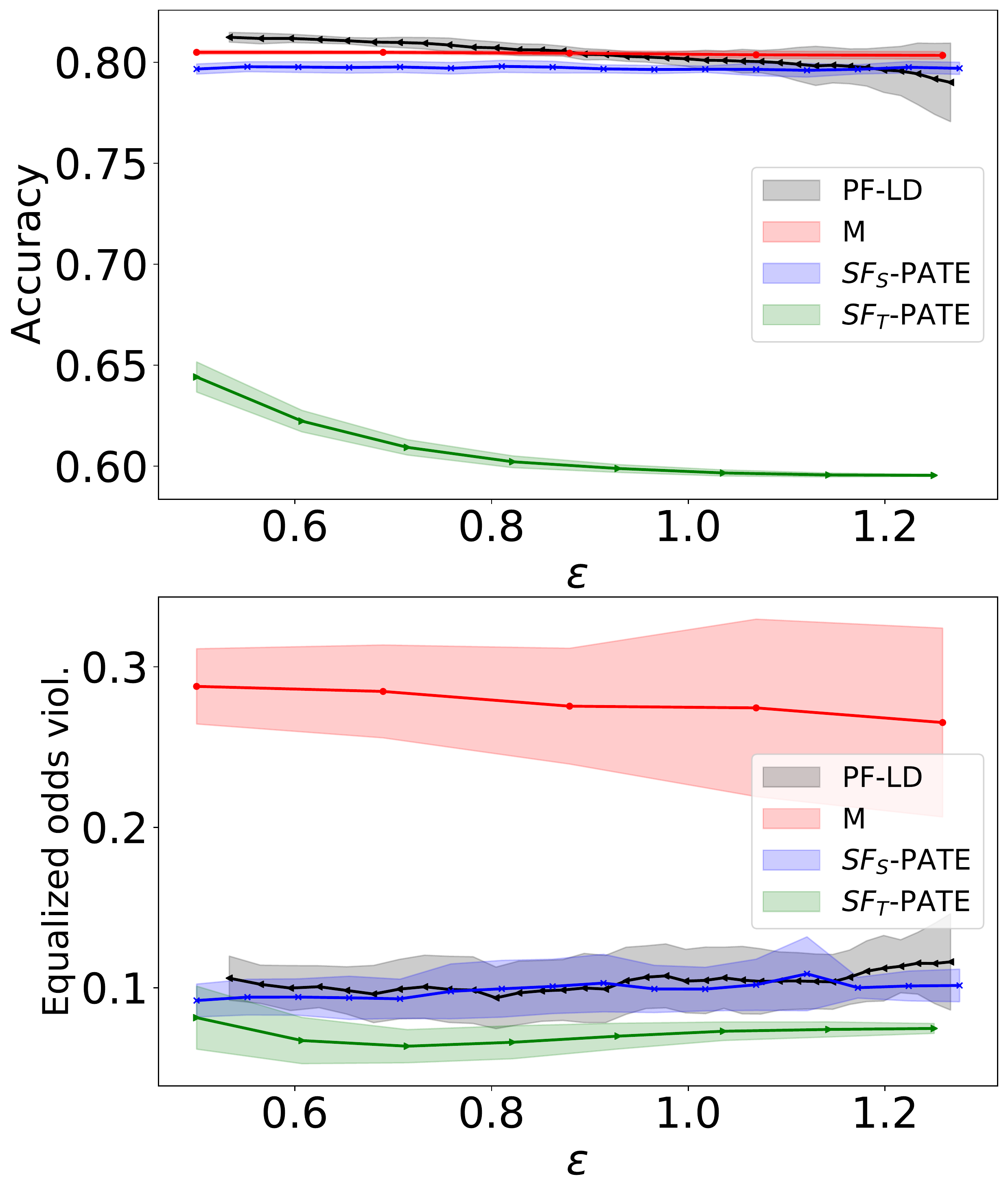}
\caption{Equalized odds fairness}
\end{subfigure}
\caption{Bank data}
\label{fig:bank_compare}
\end{figure*}

\begin{figure*}
\centering
\begin{subfigure}[b]{0.32\textwidth}
\includegraphics[width = 1.0\linewidth]{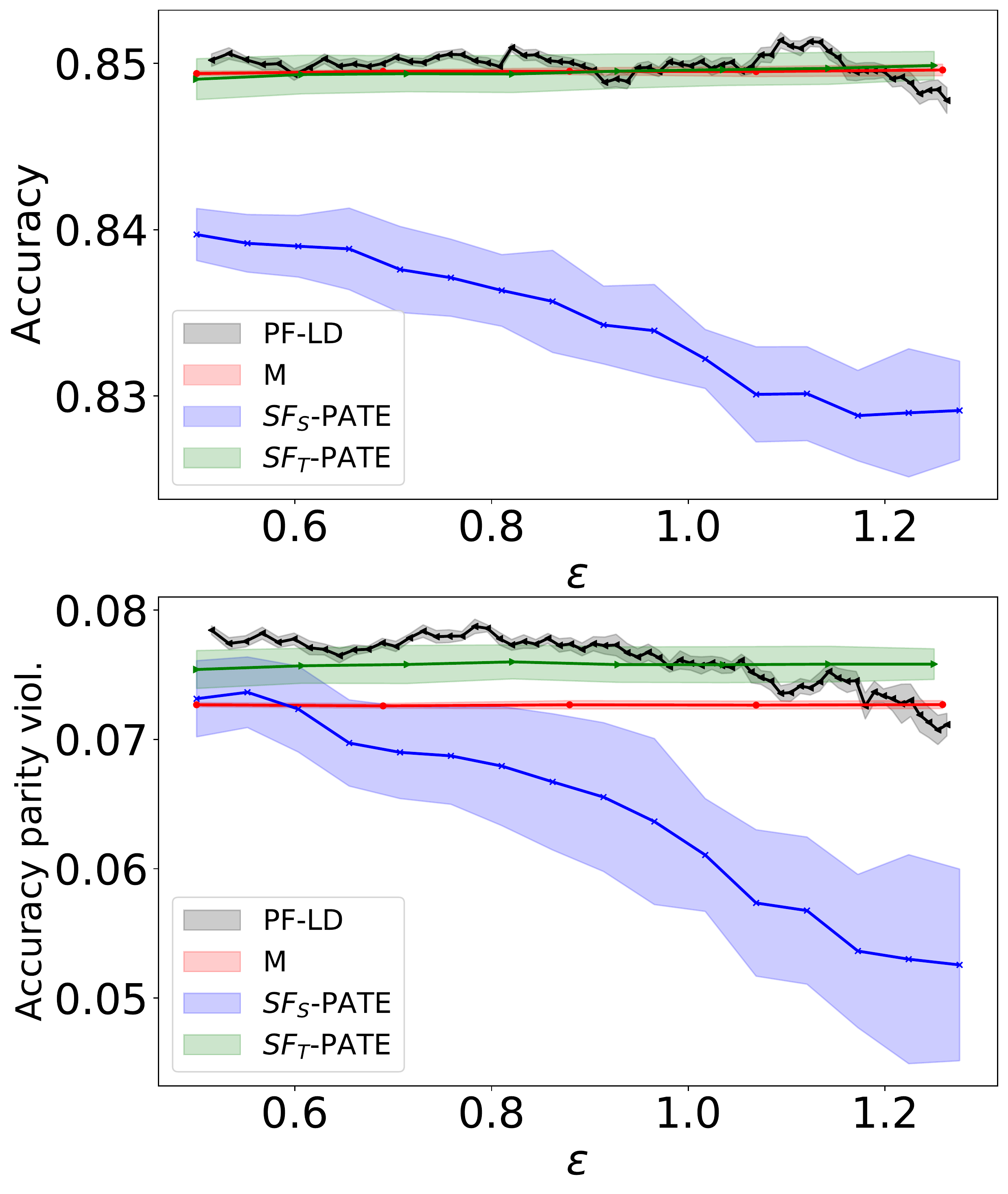}
\caption{Accuracy parity fairness}
\end{subfigure}
\begin{subfigure}[b]{0.32\textwidth}
\includegraphics[width = 1.0\linewidth]{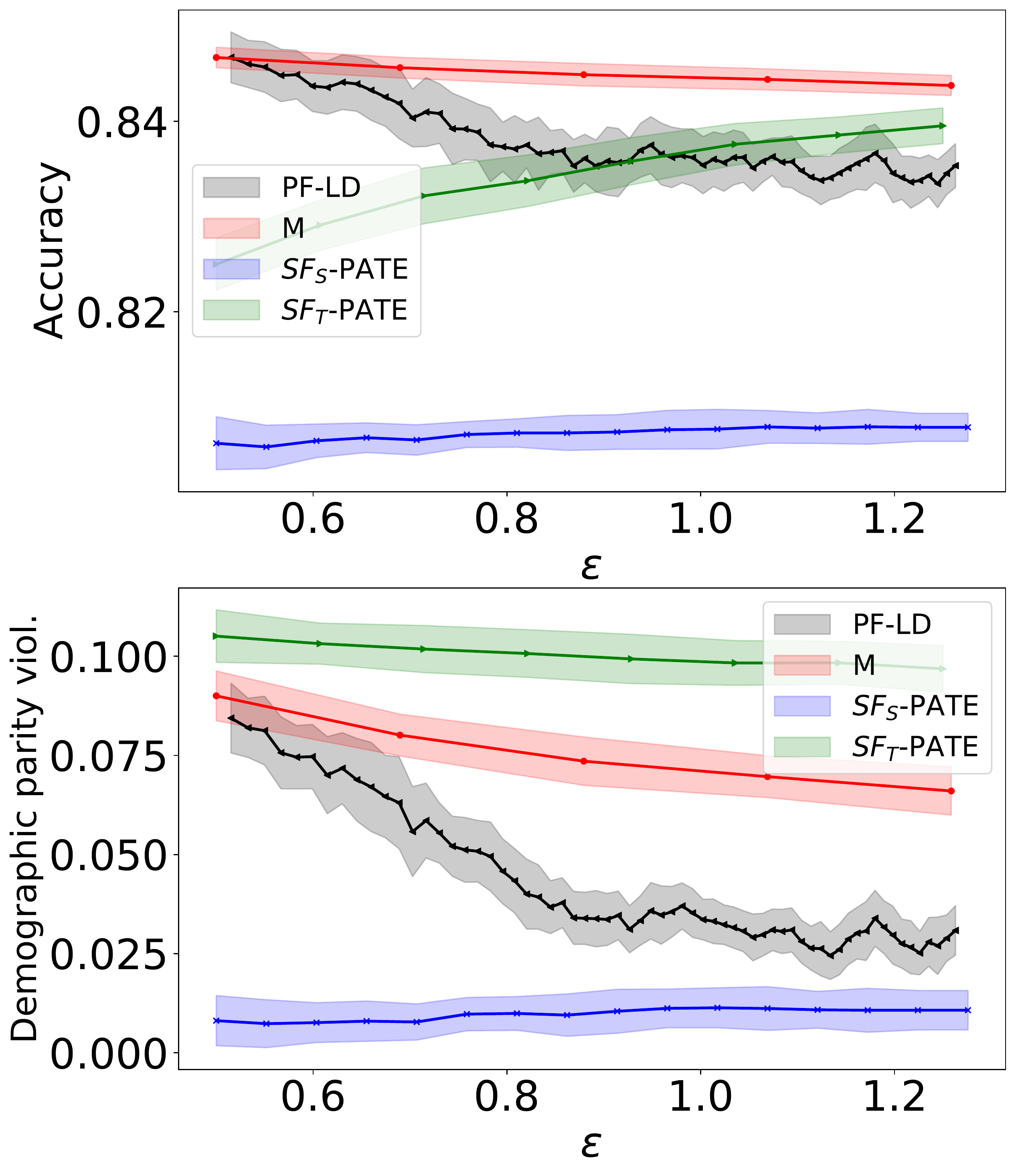}
\caption{Demographic parity fairness}
\end{subfigure}
\begin{subfigure}[b]{0.32\textwidth}
\includegraphics[width = 1.0\linewidth]{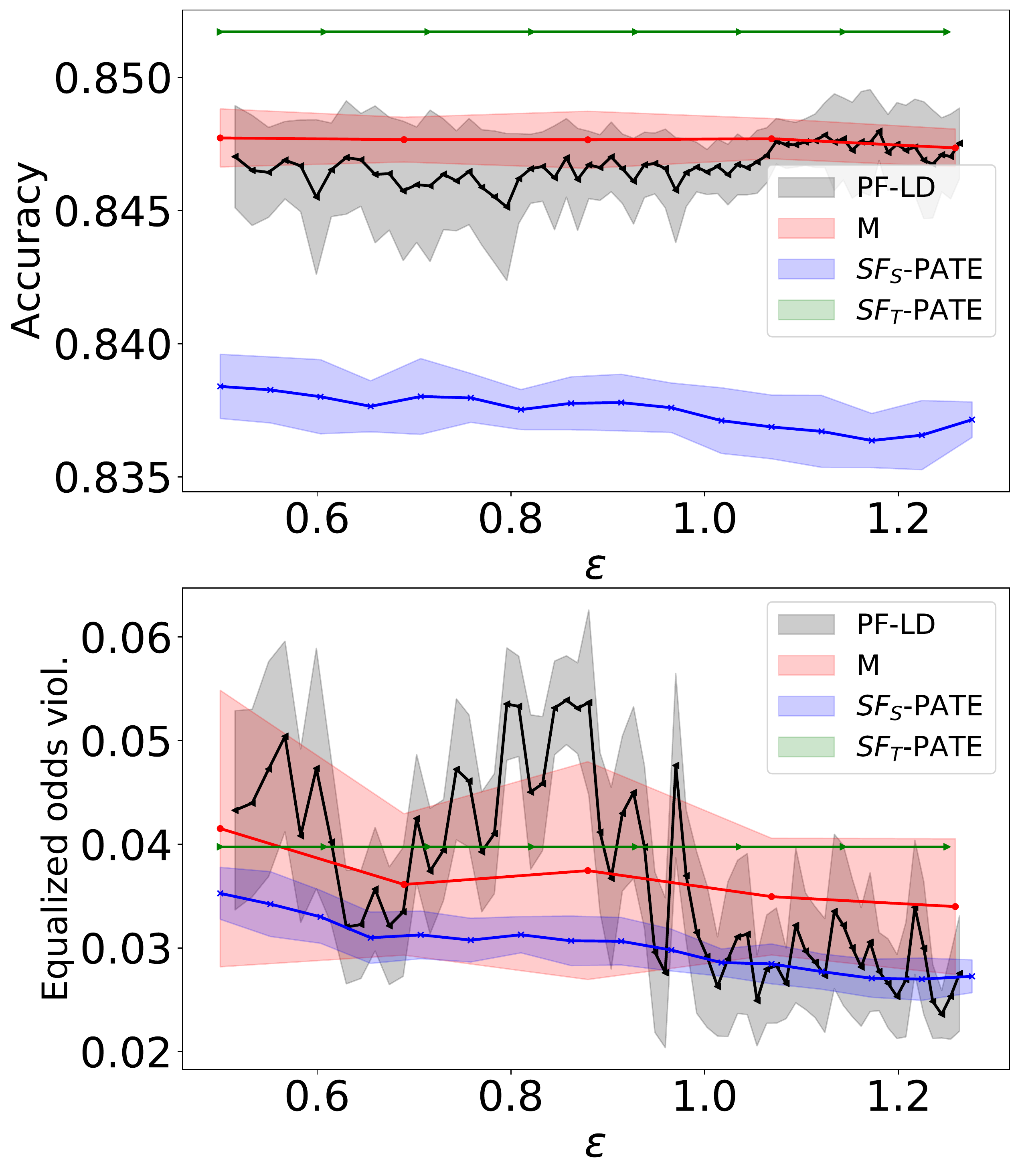}
\caption{Equalized odds fairness}
\end{subfigure}
\caption{Income data}
\label{fig:income_compare}
\end{figure*}

\begin{figure*}
\centering
\begin{subfigure}[b]{0.32\textwidth}
\includegraphics[width = 1.0\linewidth]{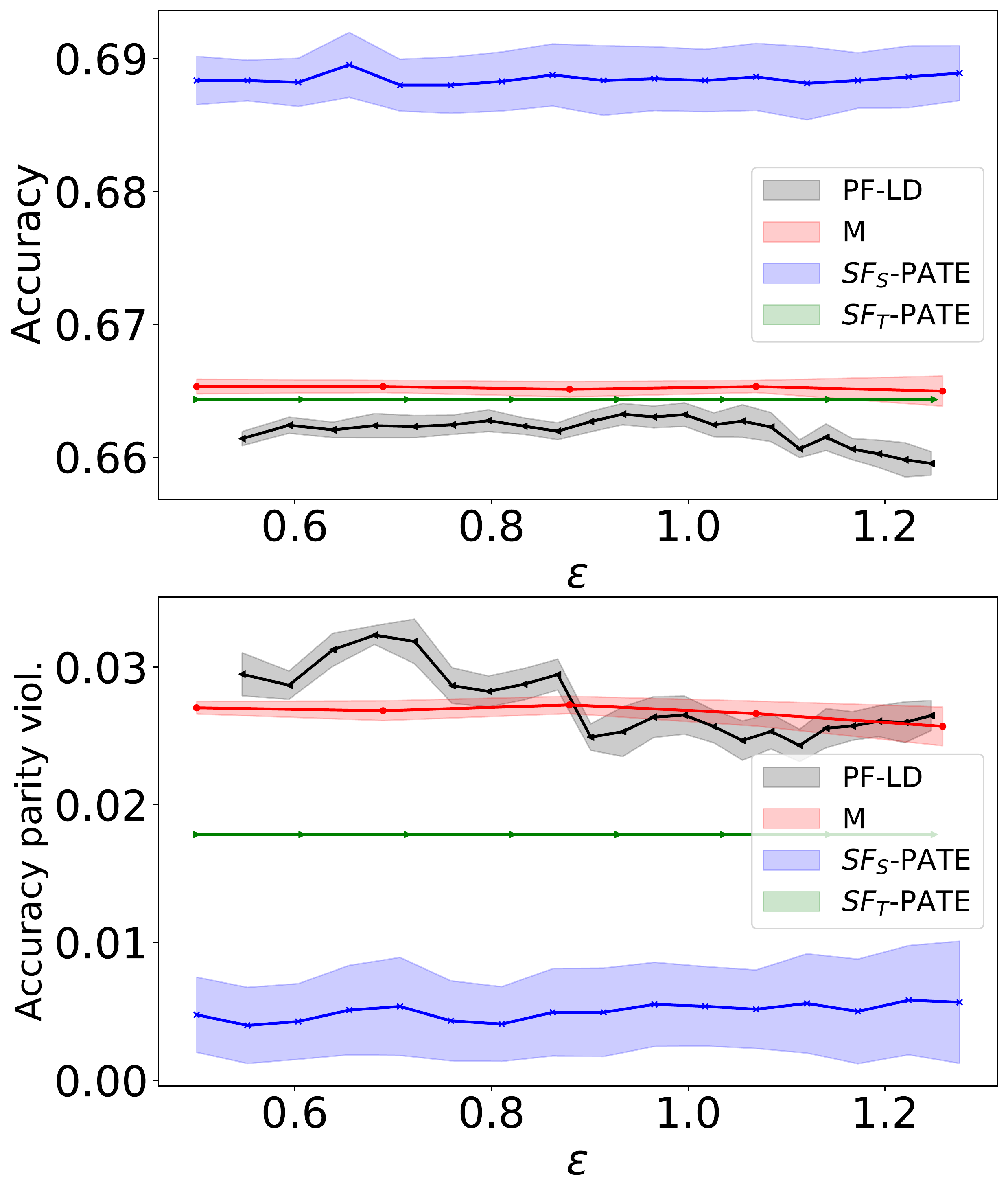}
\caption{Accuracy parity fairness}
\end{subfigure}
\begin{subfigure}[b]{0.32\textwidth}
\includegraphics[width = 1.0\linewidth]{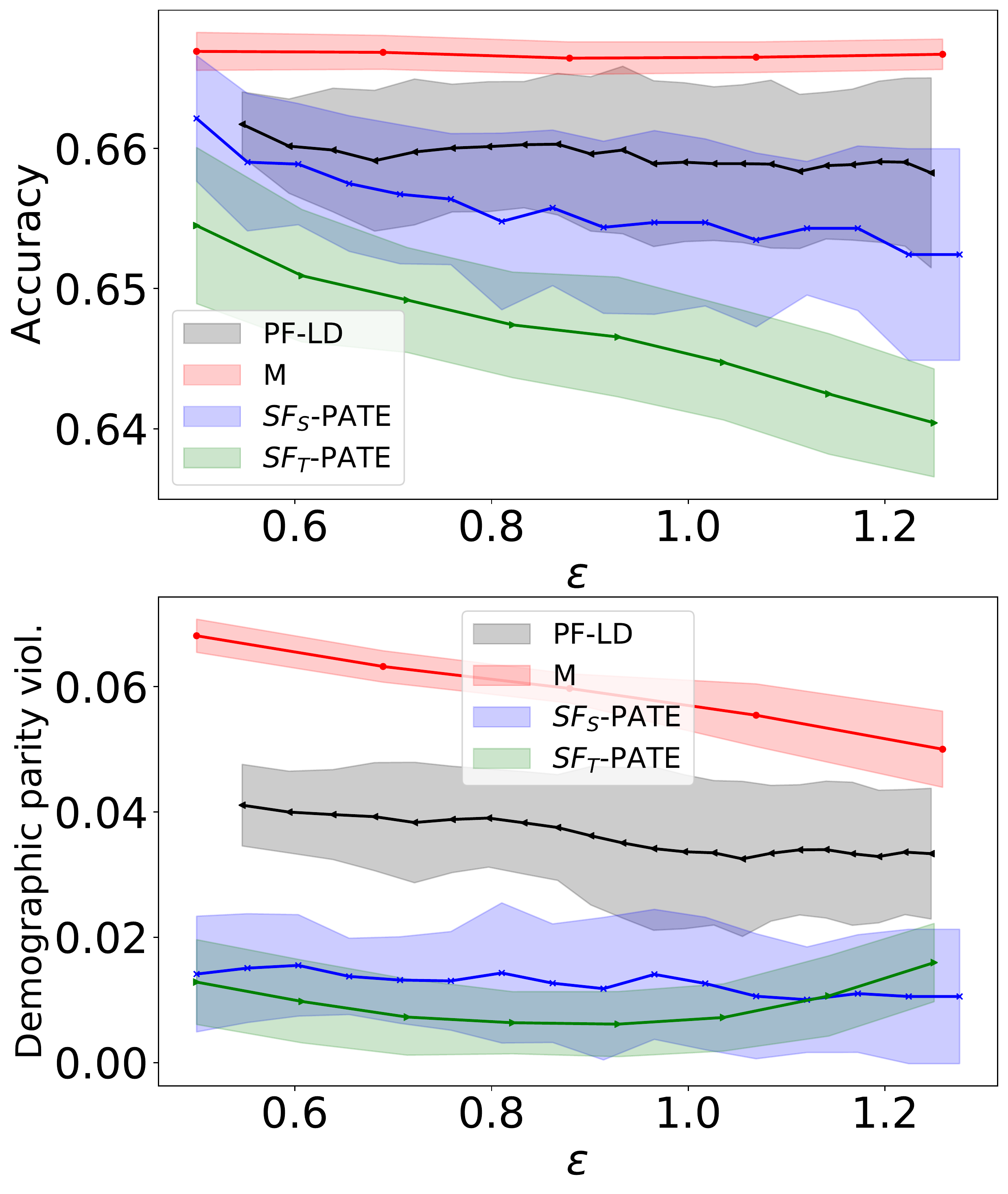}
\caption{Demographic parity fairness}
\end{subfigure}
\begin{subfigure}[b]{0.32\textwidth}
\includegraphics[width = 1.0\linewidth]{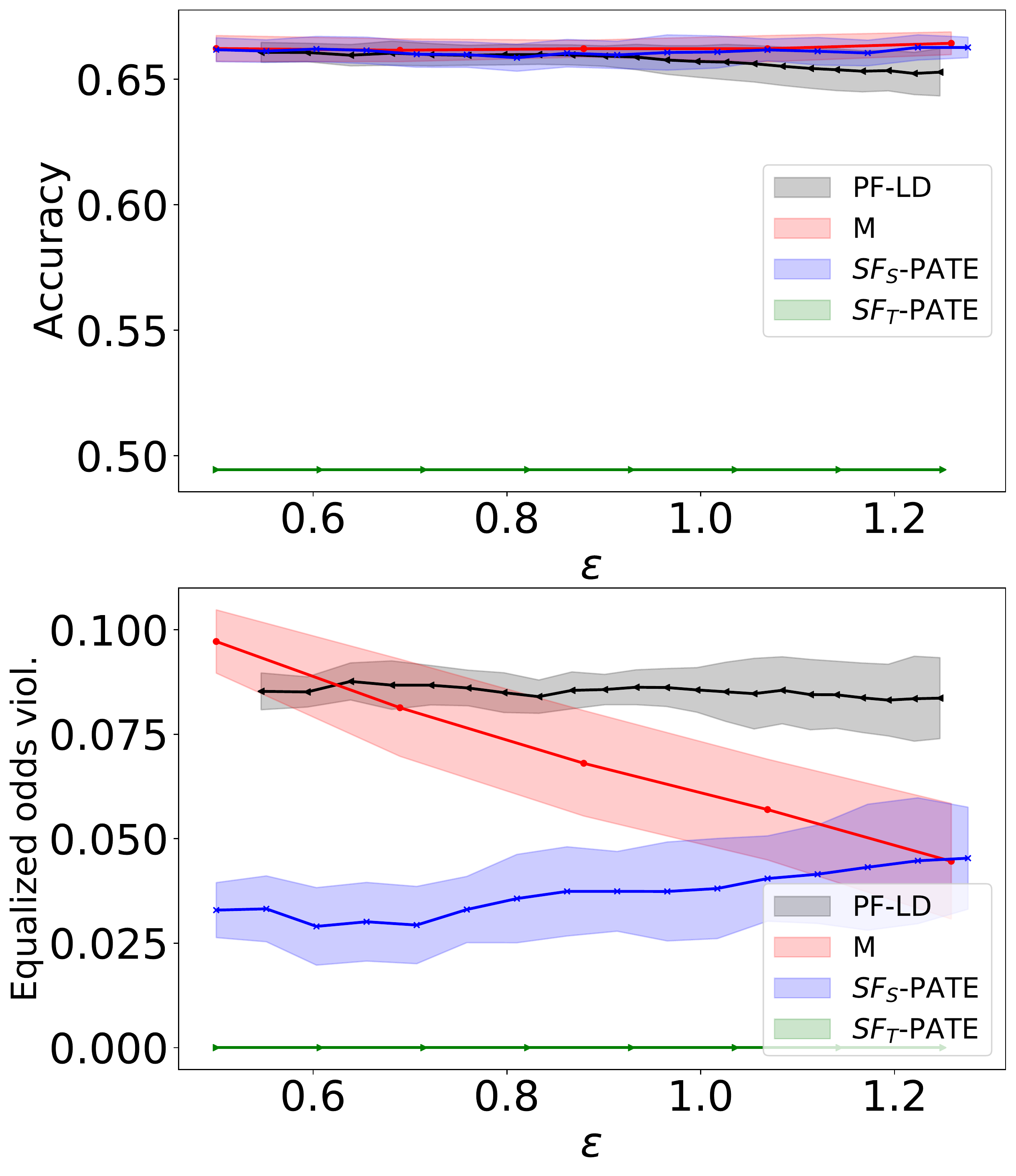}
\caption{Equalized odds fairness}
\end{subfigure}
\caption{Parkinson data}
\label{fig:parkinson_compare}
\end{figure*}

